\documentclass{article}

\usepackage{PRIMEarxiv}

\usepackage[T1]{fontenc}    
\usepackage{hyperref}       
\usepackage{url}            
\usepackage{booktabs}       
\usepackage{amsmath,amsfonts}       
\usepackage{nicefrac}       
\usepackage{microtype}      
\usepackage{lipsum}
\usepackage{fancyhdr}       
\usepackage{graphicx}       
\usepackage{algorithmic}
\usepackage{algorithm}
\hypersetup{hidelinks}
\usepackage{indentfirst}

\usepackage[caption=false]{subfig}
\newtheorem{theorem}{Theorem}
\newtheorem{remark}{Remark}
\newtheorem{definition}{Definition}
\newtheorem{assumption}{Assumption}
\newtheorem{lemma}{Lemma}

\newtheorem{proposition}{Proposition}

\newtheorem{corollary}{Corollary}

\title{Mobile Target Search with Imperfect Perception: A Partially Observable Stochastic Game Theoretical Approach

}
\makeatletter

\newcommand{\blthanks}[1]{%
  \g@addto@macro\@thanks{%
    \begingroup
      \renewcommand{\thefootnote}{}
      \footnotetext{#1}
    \endgroup
  }
}

\makeatother
\author{
Hanzheng Zhang\thanks{
Shanghai Research Institute for Intelligent Autonomous Systems, Tongji University, Shanghai, China},\quad  
Shu Liang$^{1,}$\thanks{Department of Control Science and Engineering, Tongji University, Shanghai, China
}\ ,\quad  Shuyu Liu$^2$
\blthanks{
This work was supported by the National Key Research and Development Program of China under Grant 2022YFA1004700.
Emails: \{zhanghanzheng, sliang, shuyuliu\}@tongji.edu.cn. Corresponding author: S. Liang.
}
}

\begin{document}
\maketitle

\begin{abstract}
This paper investigates mobile target search under imperfect perceptions caused by sensor limitations, malicious jamming, or communication noise. Searchers and targets operate in a grid-shaped area with bounded mobility, leading to a dynamic interplay between search and evasion. To capture this adversarial interaction under imperfect perceptions, we adopt the partially observable stochastic game (POSG) approach, which generalizes partially observable Markov decision processes (POMDPs) by incorporating target intelligence. To handle false alarms and missed detections caused by perceptual uncertainties, we propose a novel detectability concept to determine whether a search strategy guarantees eventual detection, and provide sufficient detectability criteria based on stochastic recurrence analysis. We further develop a server-assisted distributed algorithm that utilizes the aggregative potential game structure for searchers and a KL-divergence-based reduction for target prediction. Numerical simulations validate the effectiveness of the proposed algorithm and support the detectability analysis.
\end{abstract}


\section{Introduction}
Recent technological advancements in robotics have spurred research interest in mobile robots across various application scenarios such as logistics \cite{Stoyanov2013, Seder2019}, environmental monitoring \cite{Dunbabin2012, Jensen-Nau2021}, and emergency response \cite{Sung2022,Liu2025}.
Within mobile robotics, target search addresses the problem of efficiently detecting and localizing targets in unknown or dynamic environments. This capability is key to mission success in applications such as surveillance, search and rescue, and patrol operations \cite{Raj2022, Zhang2023, Alsamhi2022, Aggravi2021, Liu2025, Sung2022,Sun2025}.

Target search encompasses a range of challenges, including efficient path planning for searchers \cite{Jensen-Nau2021,Liu2025}, optimal allocation of sensing resources \cite{Hohzaki2016,woiceshyn2018vehicle}, and determining search positions over time \cite{Baylog2018,Sung2022}. These problems become particularly complex when the target is mobile and adversarial, or the environment is uncertain \cite{Kashino2020,Liu2025,Fang2025,Sung2022}. Our work focuses on multi-stage search problems for a mobile target with both imperfect sensing and limited information.

One of the main challenges in target search arises from actively evading targets, which posses obstacles in predicting the target's future positions. To this end, game-theoretic approaches, as an effective tool in predicting adversarial actions, have been widely employed in mobile target search. These methods model the search-evasion interaction as a two-player (or two-group) game between searchers and targets, yielding strategies that are effective against rational adversaries \cite{Basar1998,Zhang2024,Zheng2023,kozuno2021learning}. 
Among these, pursuit-evasion games have evolved into a distinct category, which typically assume perfect sensor perception and complete information \cite{Gregorin2017, Fang2025}. 

Beyond the challenge of actively evading targets, environmental uncertainties, and perception failures present another layer of complexity in realistic search scenarios. These issues stem from a variety of sources, including inherent sensor limitations (such as limited range and field of view \cite{Xiao2019}), dynamic environmental changes \cite{Raj2022}, and even adversarial interference such as malicious jamming \cite{jiang2021}. Such imperfections lead to incomplete, noisy, or entirely missing observations. This unreliability significantly complicates the search problem. Firstly, it introduces and perpetuates erroneous beliefs of the target, leading searchers to search in wrong areas \cite{Zhang2023}. Secondly, it forces searchers to engage in a complex trade-off between exploring the environment to reduce uncertainty and exploiting current, but possibly incorrect, beliefs to capture the target \cite{Alsamhi2022}.

Partially observable Markov decision processes (POMDPs) and partially observable stochastic games (POSGs) are widely employed to address uncertainties arising from environments and perceptions, where they model the task as a Markov process and a stochastic game, respectively. Specifically, POMDPs \cite{Hsu2008,chen2023,Zhang2025,Sung2022,jiang2021} offer a principled approach for sequential decision-making under state and observation uncertainty. 
POSG approaches \cite{Horak2019,Sun2023,Zheng2023,kozuno2021learning} additionally combine player interactions in adversarial scenarios under environmental uncertainties and sensor limitations. However, most POSG studies for target search focus on problems within the searcher coalition (e.g. \cite{kozuno2021learning}), while their application to mobile target search with false alarms and missed detections remains limited.

However, the aforementioned POSG models for mobile target search have neglected the existence of false alarms, a crucial aspect of perception. Existing works typically assume near-perfect sensing, ignoring false alarms and thereby sidestepping the potential risk of being misled caused by unreliable perceptions. Considering this, in this paper we employ a practical sensing model incorporating both false alarms and missed detections. This sensing model with missed detections and false alarms is common in imperfect search and tracking problems \cite{Baylog2018,Sung2022,Cheng2021}, and related uncertainty-aware search models have also been studied in robotic exploration and search-and-rescue scenarios \cite{Aggravi2021,Zhang2023,Alsamhi2022}. However, its application to mobile target search remains limited. Consequently, under this more challenging model, a fundamental and previously under-explored question arises: can the searcher guarantee eventual detection of the target at all? This concept of detectability, that is, whether a search strategy ensures a terminal detection state, becomes critical.

This question highlights a fundamental challenge in adversarial environments with active evasion and imperfect sensing. Under such conditions, the signal-to-noise ratio (SNR) of positive outcomes dramatically decreases compared to cases without false alarms. With a much smaller SNR, it follows from the active evasion of the target that a positive information gain cannot be guaranteed over each iteration. Consequently, conventional approaches that rely on optimizing utilities or computing Nash equilibrium strategies in Markov-based models cannot directly lead to eventual detection. The searcher's true objective, namely reaching an end state corresponding to detection, necessitates a further re-evaluation of strategy design against adversarial and perceptually limited settings. Therefore, a major contribution of this work is to bridge this gap by formally establishing an $\alpha$-detectability condition under imperfect perceptions and developing computational methods to derive strategies that provably ensure detection.

To this end, we construct the natural filtration generated by sequential belief updates, search actions, target motions, and search results, and analyze detectability as a hitting property of the posterior belief process. By applying the Borel-Cantelli lemma, we establish a sufficient condition for $\alpha$-detectability, i.e., the almost-sure eventual crossing of the detection threshold. Furthermore, for the homogeneous setting, we provide a quantitative analysis of the worst-case scenario that searchers may face, yielding explicit bounds on the detection threshold. These results not only offer insight into how detectability can be achieved, but also highlight a fundamental distinction between our problem and existing work without false alarms.

To sum up, in this paper, we explore an approach for a class of mobile target search problems incorporating both imperfect perceptions and agent mobility. In this scenario, searchers receive imperfect perceptions due to limitations in sensor abilities \cite{Sung2022, Xiao2019}, malicious jamming \cite{jiang2021}, and communication noise \cite{Sun2025}, while both searchers and the target are mobile within a constrained area. This combination prompts an investigation into search strategies to determine whether they can lead to detection and how to achieve an $\alpha$-detectable strategy. To this end, we analyze the stochastic recurrence of the posterior belief process and provide sufficient conditions for $\alpha$-detectability from both the qualitative and quantitative perspectives. Then, to overcome the computational complexities arising from exponentially large state and action spaces, we design a distributed algorithm with a central server that incorporates strategic approximations for both sides, drastically reducing the problem dimension to circumvent the curse of dimensionality inherent in POSGs.

In this view, the main contributions are as follows.
\begin{itemize}
    \item 
    We investigate a mobile target search problem under imperfect perceptions, which includes false alarms and missed detections. This generalizes existing works by incorporating target intelligence \cite{Zhang2023, Alsamhi2022, Aggravi2021, Baylog2018} and imperfect perceptions \cite{Liu2025, jiang2021,chen2023}. To handle the active target and imperfect perceptions including false alarms, we formulate the problem as a partially observable stochastic game (POSG), which provides a principled framework for reasoning under both observation errors and strategic interactions.
    
    \item 
    We establish the concept of $\alpha$-detectability to formally determine whether target detection can be guaranteed in the presence of false alarms. By leveraging the Borel-Cantelli lemma, we provide sufficient conditions for $\alpha$-detectability from both qualitative and quantitative perspectives. Compared with existing robotic search and mobile-target coordination studies where false alarms are absent \cite{Liu2025,Zhang2023,Aggravi2021}, we bridge the gap between strategy optimality and the detection guarantee, which is crucial when false alarms are present and mislead the search.

    \item 
    We design a server-assisted distributed algorithm for detectable strategies. The searcher-side problem is decomposed based on the aggregative potential game structure, while the target-side prediction is reduced through a KL-divergence reformulation based on the scalar average belief. These design choices drastically reduce computational burdens while preserving the strategic reasoning essential for effective search.
\end{itemize}

The rest of this paper is organized as follows. Section~\ref{sec-pre} summarizes the preliminaries. Section~\ref{sec-pf} formulates the search problem and reformulates it as a POSG. Section~\ref{sec-detectable} analyzes the $\alpha$-detectability using stochastic recurrence arguments. Section~\ref{sec-algorithm} presents the distributed algorithm with detailed implementations. Section~\ref{sec-exp} provides numerical simulations for illustration. Finally, Section~\ref{sec-conclusion} concludes the paper.

\textbf{Notations:} Denote the $n$-dimensional real Euclidean space by $\mathbb{R}^n$ and its $L^\infty$-norm by $\lVert \cdot\rVert_\infty$. For a set $\mathcal{A}$ with finite elements, $|\mathcal{A}|$ denotes its cardinality, and $\Delta(\mathcal{A})=\{\sum_{i=1}^{|\mathcal{A}|}k_ia_i: \sum_{i=1}^{|\mathcal{A}|}k_i=1, \ k_i\geq 0,\ a_i\in \mathcal{A}\}$ denotes the simplex over $\mathcal{A}$.

\section{Preliminaries}\label{sec-pre}

\subsection{Probability and Information Entropy}

For events $A$ and $B$, $\mathbb{P}(A)$ is the probability of $A$ occurring, and $\mathbb{P}(A|B)$ denotes the conditional probability that $A$ occurs when $B$ occurs, which satisfies $\mathbb{P}(A|B)=\frac{\mathbb{P}(AB)}{\mathbb{P}(B)}$, where $AB$ refers to both events occurring. An event $A$ is said to be a.s. (almost surely) valid if $\mathbb{P}(A)=1$. 
A binomial distribution with positive integer $K$ and $p\in[0,1]$ is denoted by $B(K,p)$, with probability mass function $\mathbb{P}(X=k) = \binom{K}{k} p^k (1-p)^{K-k}$.
For a given random variable $X$ and its probability distribution $P$, the expectation of $f$ is $\mathbb{E}_{X\sim P}[f(X)]=\sum_{X}P(X)f(X)$, which is often simplified by $\mathbb{E}[f(X)]$.

Information entropy is defined as
$H(X)=-\sum_{i=0}^{n-1} \mathbb{P}(X=i)\log_2(\mathbb{P}(X=i))$ for a random variable $X$ with possible outcomes $0, 1,\dots, n-1$. Furthermore, the expected conditional information entropy is $\mathbb{E}[H(X|Y)]=\sum_{y} \mathbb{P}(y) H(X|Y=y)$.

Specifically, the binary information entropy of $X$ is $H(X)=- p\log_2(p)-(1-p)\log_2(1-p)$, where $p$ is the probability of $X=1$. For simplification, denote the above entropy as $h(p)$ where $p$ is the corresponding positive probability. Obviously, $h(p)$ is concave and non-negative with $h(p)=h(1-p)$, $\arg\max_p h(p)=0.5$, and $h(0)=h(1)=0$.

\subsection{Filtration, Stopping Time, and Borel-Cantelli Lemma}

We provide the necessary probabilistic tools used in the detectability analysis in Section~\ref{sec-detectable}. The detailed introduction can be found in standard textbooks such as \cite{hajek2015random} and \cite{durrett2019probability}.

Let $(\Omega,\mathcal F,\mathbb P)$ be a probability space. A filtration $\{\mathcal F_t\}_{t\geq0}$ is an increasing sequence of sub-$\sigma$-algebras representing the information available over time. A random variable $\tau:\Omega\to\mathbb N\cup\{\infty\}$ is a stopping time with respect to $\{\mathcal F_t\}$ if $\{\tau\leq t\}\in\mathcal F_t$ for every $t$. The first time the posterior belief satisfies the detection criterion is treated as such a stopping time. The following (conditional) Borel-Cantelli Lemma is an effective tool to obtain a finite stopping time.

\begin{lemma}[Borel-Cantelli Lemma]\label{lem_conditional_bc}
Let $\{A_t\}_{t\geq1}$ be a sequence of events adapted to $\{\mathcal F_t\}_{t\geq0}$, with $A_t\in\mathcal F_t$. If
$$
\sum_{t=1}^{\infty}\mathbb P(A_t|\mathcal F_{t-1})=\infty,\quad \text{a.s.},
$$
then $A_t$ occurs infinitely often.
\end{lemma}

\section{Problem Formulations}\label{sec-pf}

In this section, we formulate a mobile target search problem and introduce the $\alpha$-detectability to characterize the ability to eventually detect the target. Based on this, we then present the two main problems: how to determine the $\alpha$-detectability and how to find an $\alpha$-detectable search strategy. To tackle these problems, we reformulate the search process with a multi-agent partially observable stochastic game (POSG).

\subsection{Mobile Target Search}

Consider a search problem with $m$ searchers and a mobile target in an $A_1\times A_2$ grid area $\mathcal{A}$, with the searcher set $V=\{1,\dots,m\}$. Each searcher $i$ is associated with a search coverage region $\mathcal{C}_t(v^i)$ centered at its current position $v^i_t$. While the shape of this region can be designed based on specific sensing characteristics, in this paper we assume it to be a square with a side length of $2r_s$. Formally, this region is defined as $\mathcal{C}_t(v^i) = \{a : \lVert a - v^i_t\rVert_\infty \leq r_s\}$. For brevity, the coverage is denoted by $\mathcal{C}^i_t$ for searcher $i$ and $\mathcal{C}_t=\bigcup_{i}\mathcal{C}^i_t$ for all searchers. Within this region, the searcher can detect a valid target with probability $p_d^i$ and may generate a false alarm with probability $p_f^i$. Specifically, if the target lies in an area $a$ within the search coverage $\mathcal{C}^i_t$ of searcher $i$, a positive detection occurs with probability $p_d^i$ and a negative outcome with $1-p_d^i$; otherwise, if $a$ contains no target, a false positive occurs with probability $p_f^i$ and a negative outcome with $1-p_f^i$. 
Each searcher has a mobility capacity, namely maximum speed, that constrains its feasible movement $x^i_t = v^i_t - v^i_{t-1}$. Given its previous position $v^i_{t-1}$, the corresponding action set is $X^i_t = \{x^i : \lVert x^i\rVert_\infty \leq r_{m,s},\, v^i_{t-1} + x^i \in \mathcal{A}\}$ simply denoted as $X^i_t$, where $r_{m,s}$ is the searcher's mobility. Meanwhile, the target moves evasively from $w_{t-1}$ to $w_t$ with displacement $y_t = w_t - w_{t-1}$ and its action set $Y_t = \{y : \lVert y\rVert_\infty \leq r_{m,t},\ w_{t-1} + y \in \mathcal{A}\}$, where $r_{m,t}$ is the target's mobility, simply denoted as $Y_t$.

At each time step, the target and the searchers choose their movements simultaneously, and the searchers then perform search actions. The search result at time $t$ is denoted by $o(t) = (o_a^i(t))^{i \in V}_{a \in \mathcal{C}^i_t}$, where $ o_a^i(t) = 1 $ indicates that searcher $i$ reports a target in area $a$ at time $t$, and $o_a^i(t) = 0$ indicates no detections here. Suppose that the search results $o_a^i$ ($i\in V,\ a\in \mathcal{C}^i_t$) are mutually independent.

The searchers maintain a per-area confidence regarding the possible presence of the target. Rather than modeling the target presence as a categorical random variable whose probabilities must sum to one, we treat the presence of the target in each area as an independent binary random variable. Formally, let $\{\xi_t(a)\}_{a\in\mathcal{A}}$ be a family of independent Bernoulli random variables, where $\xi_t(a)\in\{0,1\}$ indicates whether the target is in area $a$ at time $t$. Therefore, $b_a(t) = \mathbb{P}(\xi_t(a)=1)$ is the confidence (called prior belief) for $a$. 

After each search at time $t$, the searchers update this per-area confidence based on the local detection outcome. For searcher $i$ covering area $a$, let $o_a^i(t)\in\{0,1\}$ denote the binary observation, with 1 for positive and 0 for negative. The updated confidence (called posterior belief) for area $a$ is obtained via the Bayesian update rule 
\begin{equation}\label{eq_posterior}
\begin{aligned}
b_a(t| o_a^i(t)=1) &= \frac{p_d^i b_a(t)}{p_{pd}^i(a,t)}, \\
b_a(t| o_a^i(t)=0) &= \frac{(1-p_d^i)b_a(t)}{1-p_{pd}^i(a,t)},
\end{aligned}
\end{equation}
where $p_{pd}^i(a,t) = p_d^i b_a(t) + p_f^i(1-b_a(t))$ is the prior probability of a positive detection result in $a$ by searcher $i$. To simplify the notation, we denote the posterior belief by $\tilde{b}_a(t)$ and write $p_{pd}^i$ for $p_{pd}^i(a,t)$ for convenience.

\begin{remark}\label{rmk_dis}
The design of the variable $b_a(t)$, which is modeled as an independent Bernoulli parameter that the sum across areas is not constrained to $1$, is intentional. Although a normalized Bayesian filter can be formulated for a single target, false alarms would couple all areas through the global normalization term, so that a noisy report in one area redistributes probability mass over the entire grid. In contrast, the independent Bernoulli model updates each area locally and is therefore more robust and computationally tractable under simultaneous false alarms and missed detections.
\end{remark}

We give the following assumption regarding false alarms and positive detections, which is standard in imperfect detection and search models \cite{Baylog2018,Cheng2021,Sung2022} and ensures a positive signal-to-noise ratio (SNR) as a prerequisite for detection.
\begin{assumption}\label{assumption} For each search agent $i$, the probability of false alarms and positive detections satisfies $0\leq p_f^i<p_d^i\leq 1$.
\end{assumption}

After Bayesian updating based on search results, the searchers predict the next location of the target $w_{t+1}$ based on both the posterior belief $\tilde{b}(t)$ and the expected target's strategy $\pi_{t+1}\in\Delta(Y_{t+1})$, shown as follows.
\begin{equation}\label{eq_transition}
    b_a(t+1) = \sum_{a' \in \mathcal{A} ,\, a' + y = a} \tilde{b}_{a'}(t)\tilde{\pi}_{t+1}^{a'}(y),
\end{equation}
where $\tilde{\pi}_{t}^{a}$ denotes the feasible strategy at area $a$ subject to boundary constraints: for boundary areas, it is obtained by restricting $\pi_t$ to feasible moves and renormalizing, while for interior areas $\tilde{\pi}_t^a=\pi_t$.

In a multi-agent search task, the execution of search decisions is inherently distributed, as each searcher operates from a distinct physical location. Nevertheless, maintaining a consistent global belief and emulating the target's motion require aggregating observations from all searchers, tasks that are naturally centralized. To reconcile this tension, 
we adopt a server-assisted distributed architecture. The central server aggregates all detection reports, updates the global posterior belief via~\eqref{eq_posterior}, predicts the next prior belief through~\eqref{eq_transition}, and broadcasts it to all searchers. Each searcher then exchanges intended coverages with neighboring agents whose sensing regions may overlap and solves a local movement optimization problem based on the global belief and local coordination information. This structure preserves consistent belief synchronization while avoiding a fully centralized search planner.

To measure strategies under imperfect perceptions, we employ information entropy as a metric to quantify search performance, which refers to the uncertainty of a system. As search results in our framework are binary, the information entropy of the prior belief can be expressed as
$$
H(\xi_t)= -\sum_{a\in \mathcal{A}}(b_a(t)\log_2 b_a(t)+ (1-b_a(t))\log_2 (1-b_a(t))),
$$
where $\xi_t$ is the $|\mathcal{A}|$-dimensional binary random variable defined above. Therefore, the posterior (conditional) entropy is defined as
$H(\xi_t|o(t))=\sum_{a\in \mathcal{A}}H(\xi_t|o_a(t)).$
As in the preliminaries, we take $h(p)$ to represent the binary entropy $H(\xi)$ with $\mathbb{P}(\xi=1)=p$. Also, for simplification, we abuse the notation $h$, where $h(b(t))$ refers to the total entropy given the belief vector $b(t)$ and $h(b_a(t))$ denotes the entropy associated with its belief value $b_a(t)$. Consequently, the posterior entropy can be compactly written as $h(\tilde{b}(t))$.

\begin{definition}[detection threshold]\label{def-searchthreshold}
The searchers are said to successfully detect the target if for a given $\alpha\geq0$ (called \emph{detection threshold}), the posterior belief $\tilde{b}(t)$ satisfies: (i) $\max_a h(\tilde{b}_a(t))\leq\alpha$; (ii) there exists exactly one area $a^\ast\in\mathcal{A}$ such that $\tilde{b}_{a^\ast}(t)>0.5$.
\end{definition}

Based on the above definition, we propose a novel concept to judge whether searchers operating under a specific strategy or a specific scenario possess the ability to detect the target.
\begin{definition}[$\alpha$-detectability]\label{def-detectable}
Given a detection threshold $\alpha$, a search strategy is called \emph{$\alpha$-detectable} if, for any adversarial strategy, with probability $1$ there exists a finite time $\tau$ at which the searchers successfully detect the target as in Definition~\ref{def-searchthreshold}. A search game is called \emph{$\alpha$-detectable} if there exists at least one $\alpha$-detectable search strategy for the searchers.
\end{definition}

\begin{remark}
Existing works have provided several criteria to judge detection. Most studies that do not consider false alarms~\cite{Zhang2023, Alsamhi2022, Aggravi2021,Sun2023} treat the target as detected once the positive detection probability exceeds a given threshold, which is reasonable when $p_f=0$. When false alarms are present, \cite{Cheng2021} introduced the cumulative number of positive outcomes to confirm detection, assuming a static target so that repeated observations can overcome noise. However, this criterion is not suitable for our mobile setting.
\end{remark}

For the proposed mobile target search, we first explore the detectability criteria. To ensure target detection, we must answer whether a strategy or a game is $\alpha$-detectable. Unlike works with reliable sensors that an optimal strategy often implicitly ensures eventual detection \cite{Zhang2023, Alsamhi2022, Aggravi2021}, our work has to answer this due to both the engagement of evasive target and low SNR of positive detection. The problem is summarized below.
    
\textit{Problem 1 (detectability criterion)}: How to judge whether a strategy or a search game is $\alpha$-detectable or not?

Next, we investigate how to obtain an $\alpha$-detectable strategy. The main challenge is the computational difficulties arising from exponentially large state and action spaces. This prompts us to design an efficient algorithm with certain approximations to both make the computation of search strategies tractable and ensure effective search. The problem is as follows.

\textit{Problem 2 (detectable strategy computation)}: How to design a distributed algorithm for an $\alpha$-detectable search strategy?

\subsection{POSG Reformulations}

To handle the above problems, we reformulate the mobile target search as a multi-agent POSG, which is effective in mobile target search.


For searchers, their objective is to reduce the uncertainty of search for detecting the target, i.e., to minimize the posterior information entropy after a search. Consequently, we propose the utility based on the posterior information entropy

\begin{equation}\label{eq_utility-entropy}
u(x,\pi|\tilde{b}(t-1))=\sum_{a\in \mathcal{A}} \mathbb{E}[h(\tilde{b}_a(t))| x,\pi,\tilde{b}(t-1)],
\end{equation}
where $h(\tilde{b}_a(t))$ is the expected posterior information entropy and $\mathbb{E}[\cdot|x,\pi]$ is its expectation before implementation of the search and evasion as 
\begin{equation}\label{eq_expected-entropy}
\begin{aligned}
\mathbb{E}[h(\tilde{b}_a(t))|x,\pi]
=\sum_{o_a\in O_a(t)}\mathbb{P}(o_a|b_a(t|\pi))h(b_a(t|o_a(t)=o_a,\pi)),\end{aligned}
\end{equation}
with $O_a(t)$ referring to all feasible outcomes with respect to the current position, $b_a(t|\pi)$ (i.e., $b_a(t)$) is the expected prior belief according to~\eqref{eq_transition}, $b_a(t|o_a,\pi)$ is the expected posterior belief based on~\eqref{eq_posterior}, and $\mathbb{P}(o_a|b)$ is the subjective probability of outcome $o_a$ according to belief $b_a(t)$, which can be obtained by
$$
\mathbb{P}(o_a|b_a(t))=\prod_{\{i:a\in\mathcal{C}^i_t\}} \mathbb{P}(o_a^i|b_a(t)),
$$
with $\mathbb{P}(o_a^i=1)=p^i_{pd}(a,t)$ and $\mathbb{P}(o_a^i=0)=1-p^i_{pd}(a,t)$, which can be computed via $b_a(t)$.

Conversely, the target's strategy aims to maximize the uncertainty, which corresponds to maximizing the expected posterior information entropy $u$.

Therefore, with both sides' strategies $(\sigma, \pi)$ and the prior estimation $b(t)$, the expected payoff with discount $\xi$ is 
\begin{equation}\label{eq_expected-utility} 
\begin{aligned}
    U(\sigma,\pi|\tilde{b}(t-1))=\mathbb{E}_{x\sim\sigma_t}[u(x,\pi_t)]+\sum_{\tau=1}^\infty \xi^\tau \mathbb{E}_{x\sim\sigma_{t+\tau},\tilde{b}(t+\tau-1)}[u(x,\pi_{t+\tau}|\tilde{b}(t+\tau-1))].
\end{aligned}
\end{equation}

\begin{remark}
Different from existing works \cite{Zhang2023, Alsamhi2022, Aggravi2021}, our work takes false alarms into account and evaluates performance by the metric of information entropy, where we leverage the information embedded in negative outcomes to refine our search strategy. Moreover, this setting is a generalization of the aforementioned models in \cite{Zhang2023, Alsamhi2022, Aggravi2021} when the probability of false alarms $p_f^i\equiv0$, since the probability of detection and the information entropy share an extremal point. 
\end{remark}

On this basis, the agents participate in a partially observable stochastic game (POSG) denoted by $\mathcal{G}=\langle I,\mathcal S,O,X\times Y, T, \Xi, U, b_0\rangle$, where 
\begin{enumerate}\renewcommand\labelenumi{(\roman{enumi})}
    \item $I=V\cup\{\mathrm{target}\}$ is a set of players consisting of $m$ searchers $\{1,\dots,m\}$ and a target.
    \item $\mathcal {S}(t)=\{w_t\}_{w_t \in \mathcal{A}}$ is a set of states referring to the position of the target from the searchers' perspective.
    \item $ O(t)=\{o_a^i(t)\}_{a\in\mathcal{C}^i_t}^{i\in V}$ is the observation set that refers to all possible observations of searchers.
    \item $X_t \times Y_t$ is the joint action set, where $X_t = X^1_t \times \cdots \times X^m_t$ is the set for the searchers. Here, $x^i_t \in X^i_t$ and $y_t \in Y_t$ denote the feasible movement actions starting from the current positions $v^i_t$ and $w_t$, for searcher $i$ and the target, respectively. $\sigma^i_t\in \Delta(X^i_t)$ and $\pi_t\in \Delta(Y_t)$ are the mixed strategies of searchers and the target, respectively.
    \item $T: \mathcal S\times \Delta(Y)\times \mathcal S \to [0,1]$ is the transition probability function satisfying $T(w_t,\pi_t, w_{t+1})=\mathbb{P}(w_{t+1}|w_t,\pi_t)=\pi_t(w_{t+1}-w_t)$.
    \item $\Xi: \mathcal S\times O\to \mathbb{R}$ is the observation distribution function, with $\Xi(o^i_a(t)|w_t)=p_d^i$ if $w_t=a$, and with $\Xi(o^i_a(t)|w_t)=p_f^i$  otherwise.
    \item $U$ is the utility function as in~\eqref{eq_expected-utility}.
    \item $b_0\in [0,1]^{|\mathcal{A}|}$ is an initial belief.
\end{enumerate}


In the context of search games modeled between searchers and a target, the Nash equilibrium (NE) provides a key concept for analyzing their strategic interactions. When reaching a NE, no player can unilaterally change their strategies for a better expected payoff. Technically, we give the definition as follows.
\begin{definition}[Nash equilibrium]
A strategy pair $(\sigma^\ast,\pi^\ast)$ is a Nash equilibrium (NE) for the search game if, for any given $b_0$, the following holds for each feasible $\sigma$ and $\pi$,
$$
U(\sigma^\ast,\pi^\ast|b_0) \leq U(\sigma,\pi^\ast|b_0),$$
$$U(\sigma^\ast,\pi|b_0)\leq U(\sigma^\ast,\pi^\ast|b_0).
$$
\end{definition}

The existence of NE is guaranteed by the compactness and closure of the action sets, as established in the following result.

\begin{lemma}
Under Assumption~\ref{assumption}, there exists a NE of the POSG $\mathcal G$.
\end{lemma}

The above POSG framework provides a theoretical foundation for multi-agent search games. In practice, each searcher only considers its local utility over $\bar{\mathcal{C}}_t(v^i)=\{a: \lVert a-v^i_t\rVert_\infty\leq r_s+r_{m,s}\}$:
$$u_i=\sum_{a\in \bar{\mathcal{C}}_t(v^i)}\mathbb{E}[h(b_a(t)|o(t))],$$
where the expected utility in~\eqref{eq_expected-utility} is restricted accordingly. This distributed structure aligns with the partial observability inherent in the POSG. Although agents share information through the central server, their decision-making and utility calculation remain localized to their observable regions.

\begin{remark}\label{rmk_ne_imperfect}
Since the POSG formulated above is zero-sum, a Nash equilibrium coincides with a sequential minimax strategy. Moreover, since the game is imperfect-information, predicting the target's next move and inferring its past move are equivalent from the searcher's information set. With this observation, we adopt a sequential decision model in algorithm designs in Section~\ref{sec-algorithm} to simplify computation.
\end{remark}

Based on the properties of zero-sum games, the following lemma links equilibrium and detectability, showing that solving for an equilibrium yields a sound approach for an $\alpha$-detectable strategy.

\begin{lemma}\label{lem_detect}
Let $(\sigma^\ast,\pi^\ast)$ be a Nash equilibrium. For a given threshold $\alpha$, the searchers' equilibrium strategy $\sigma^\ast$ achieves successful target detection against $\pi^\ast$ in the sense of Definition~\ref{def-searchthreshold}, if and only if the corresponding game is $\alpha$-detectable.
\end{lemma}

Lemma~\ref{lem_detect} motivates the equilibrium-based approach. For Problem 1, it justifies analyzing whether equilibrium-induced beliefs satisfy the definition of $\alpha$-detection. For Problem 2, it directs the design of equilibrium-seeking algorithm as a sound approach for an $\alpha$-detectable strategy.

\section{Detectability}\label{sec-detectable}

In this section, we investigate the detectability of our search problem. We start from the detectability criterion, to explore when a game is $\alpha$-detectable. Besides, we study a homogeneous case to give quantitative insight to determine whether a strategy or a game is $\alpha$-detectable.

\subsection{$\alpha$-Detectability Analysis: General Cases}

To establish a detectability criterion, we first investigate the mechanism through which players influence information entropy, namely, utility functions. We start from the searcher's side to show that the information entropy can be systematically decreased.

\begin{lemma}\label{lem_gain}
Suppose Assumption~\ref{assumption} holds. For any searched area $a \in  \mathcal{C}_t$ and belief $b_a(t) \in (0,1)$, the expected information gain is strictly positive. That is,
\begin{equation}\label{eq_searcher-information-gain}
h(b_a(t))>\mathbb{E}_{o_a}[h(b_a(t)|o_a)].
\end{equation}
\end{lemma}

\noindent\textbf{Proof.}
According to the mutual independence of the search results, each observation $o_a=(o_a^1,o_a^2,\dots)$ can be regarded as a cumulative result of finite single searches.
Consequently, to show the strict increase in the expected information gain of $o_a$, we only have to prove that the expected information gain after a single search $o_a^i$ is strictly positive.

Recalling the computation of the expected posterior entropy~\eqref{eq_expected-entropy} 
and combining with Bayesian updates~\eqref{eq_posterior}, 
\begin{equation*}
\begin{aligned}
\mathbb{E}[b_a(t|o_a^i)]&=p_{pd}^ib_a(t|o_a^i=1)+(1-p_{pd}^i)b_a(t|o_a^i=0)\\
&=p_{pd}^i\frac{p_d^i b_a(t)}{p_{pd}^i }+(1-p_{pd}^i)\frac{(1-p_d^i)b_a(t)}{1-p_{pd}^i}
= b_a(t).
\end{aligned}
\end{equation*}
Since the information entropy $h(p)$ is concave with respect to $p$ and $b_a(t|o_a^i=1)> b_a(t|o_a^i=0)$, combining with Jensen's inequality, $$
\mathbb{E}[h(b_a(t)|o_a^i)]
<h(p_{pd}^ib_a(t|o_a^i=1)+(1-p_{pd}^i)b_a(t|o_a^i=0))
=h(b_a(t)),
$$
which validates~\eqref{eq_searcher-information-gain} and completes the proof. 
\hfill $\square$

This lemma demonstrates that search actions always provide positive information gain, which leads to a decrease in the uncertainty, i.e., decrease in the utility. 

In contrast to searchers, the mobile target consistently escapes the searchers, which brings uncertainty to the game, that is, an increase in utility. Specifically, for the mixed strategy, evasion always increases the information entropy.

\begin{lemma}\label{lem_target}
If the target adopts a mixed strategy and the posterior beliefs are not identical, for any given $t>0$, the posterior estimate $\tilde{b}(t)$ is worse than the prior estimate $b(t+1)$ in terms of entropy. That is,
\begin{equation}\label{eq_infor_loss}
 h(\tilde{b}(t))<h(b(t+1)).
\end{equation}
A lower bound for~\eqref{eq_infor_loss} is given by the value obtained under condition $b_a(t+1)=\frac{1}{|\mathcal{A}|}\sum_{a'\in\mathcal{A}} \tilde{b}_{a'}(t)$, although the corresponding $\pi$ is typically unattainable, implying that the true minimum exceeds this bound.
\end{lemma}

\noindent\textbf{Proof.}
By the concavity of $h$,
$$
\begin{aligned}
h(b(t+1))
&=\sum_{a\in \mathcal{A}}h\left(\sum_{a'+\Delta y=a,\,a'\in \mathcal{A}} \tilde{b}_{a'}(t)\tilde\pi_t^{a'}(\Delta y)\right)\\
&\geq\sum_{a\in \mathcal{A}}\sum_{a'+\Delta y=a,\,a'\in \mathcal{A}}h(\tilde{b}_{a'}(t))\tilde\pi^{a'}_t(\Delta y)\\
&=\sum_{a'\in \mathcal{A}} h(\tilde{b}_{a'}(t))\sum_{\Delta y\in w(a')}\tilde\pi^{a'}_t(\Delta y)=\sum_{a\in \mathcal{A}} h(\tilde{b}_a(t)),
\end{aligned}
$$
where $\tilde{\pi}_t^a$ is the feasible strategy considering the boundary restriction. Because $\pi_t$ is a mixed strategy, $\tilde\pi^{a'}_t$ are non-degenerate for at least some $a'$ with $\tilde{b}_a(t)>0$. Hence, Jensen's inequality is strict for those terms, and consequently, $h(b(t+1))>\sum_{a\in \mathcal{A}}h(\tilde{b}_a(t))$. 

Note that the sum of beliefs remains unchanged, i.e., $\sum_{a} \tilde{b}_{a}(t)=\sum_a b_a(t+1)$. To maximize the prior entropy for strategy $\pi$, based on Jensen's inequality,
$$
\sum_{a\in \mathcal{A}} h(b_a(t+1)|\pi)\leq|\mathcal{A}| h\left(\frac{1}{|\mathcal{A}|}\sum_{a\in\mathcal{A}}\tilde{b}_a(t)\right),
$$
which proves the lower bound given in the lemma.
\hfill$\square$

Lemmas~\ref{lem_gain} and~\ref{lem_target} characterize the static effects of one search--evasion step: search observations reduce the expected uncertainty of searched areas, while target motion redistributes posterior belief and may increase uncertainty. However, detectability is not a one-step entropy property. It requires that the posterior belief process reaches, in finite time, a belief configuration satisfying both the entropy threshold and the unique-peak requirement in Definition~\ref{def-searchthreshold}.


\begin{theorem}\label{thm_martingale}
Suppose Assumption~\ref{assumption} holds. Given a search equilibrium $\sigma^\ast$, assume that, for every target strategy $\pi$,
\begin{equation}\label{eq_general_condition} 
\sum_{t=1}^{\infty} \mathbb P( \tilde{b}(t)\in\mathcal{B}_{d} | \mathcal{F}_{t-1} ) = \infty,\  \text{a.s.}
\end{equation}
where $\mathcal{B}_d$ is the detection belief set consisting of all posterior belief vectors satisfying the conditions in Definition~\ref{def-searchthreshold}. Then $\sigma^\ast$ is an $\alpha$-detectable strategy, and the search game is $\alpha$-detectable.
\end{theorem}

\noindent\textbf{Proof:}
Let $A_t=\{\tilde{b}(t)\in\mathcal{B}_d\}$. Since $\tilde{b}(t)$ is generated by the belief update, target motion, and search results up to time $t$, $A_t$ is adapted to the natural filtration of the search process. By the Borel-Cantelli lemma, \eqref{eq_general_condition} implies that $A_t$ occurs infinitely often almost surely. Hence, with probability one, there exists a finite time $t$ such that $\tilde{b}(t)\in\mathcal{B}_d$. Therefore, with Lemma~\ref{lem_detect}, $\sigma^\ast$ is $\alpha$-detectable against every admissible target strategy. Since such a strategy exists, the search game is $\alpha$-detectable.
\hfill$\square$

Theorem~\ref{thm_martingale} provides a qualitative recurrence criterion for detectability. It shows that if the posterior belief process has non-summable conditional opportunities to enter $\mathcal B_d$, then $\mathcal B_d$ is reached in finite time almost surely. Since the conditional probability in~\eqref{eq_general_condition} depends on sensing probabilities, overlapping coverages, target motion, and belief redistribution, it is difficult to evaluate directly in heterogeneous settings. We therefore next consider homogeneous searchers, where the belief dynamics can be reduced to an average-belief analysis and explicit sufficient conditions can be derived.

\subsection{$\alpha$-Detectability Analysis: Homogeneous Agents}

For homogeneous searchers, actions can be equivalently described through their induced sensing coverage. Let $\lambda^i(a)\in\{0,1\}$ indicate whether searcher $i$ covers area $a$, and let $\lambda(a)=\sum_{i\in V}\lambda^i(a)$ denote its aggregate coverage count. Thus, the observation distribution and the Bayesian update in an area depend on $\lambda(a)$ rather than on the identities of the covered searchers.

\begin{lemma}\label{lem_homo_prob}
In the homogeneous setting, for a covered area $a$ with coverage count $\lambda(a)\geq1$, let $k\in\{0,\ldots,\lambda(a)\}$ be the number of
positive outcomes. Then,
\begin{equation*}
\mathbb P(o_a(t)=k| w_t,\lambda(a))
=\begin{cases}
P_D^{\lambda(a)}(k), & w_t=a,\\
P_F^{\lambda(a)}(k), & w_t\neq a,
\end{cases}
\end{equation*}
where
$
P_D^{\lambda(a)}(k)=\binom{\lambda(a)}{k}p_d^k(1-p_d)^{\lambda(a)-k},
P_F^{\lambda(a)}(k)=\binom{\lambda(a)}{k}p_f^k(1-p_f)^{\lambda(a)-k}
$
The subjective probability of observing $k$ positives is $P_{PD}^{\lambda(a)}(k)=b_a(t)P_D^{\lambda(a)}(k)+(1-b_a(t))P_F^{\lambda(a)}(k)$, and the posterior belief is $\tilde{b}_a(t| k)=P_D^{\lambda(a)}(k)b_a(t)/P_{PD}^{\lambda(a)}(k)$.
\end{lemma}

To obtain an explicit sufficient condition, we impose the following homogeneous assumption.

\begin{assumption}\label{assumption_2}
At every time step, $|\mathcal{C}(t)|=C$ and each searched area has coverage count $\lambda$. Moreover, for a search strategy $\sigma$, $\mathbb P(w_t\in\mathcal{C}(t)|\mathcal{F}_{t-1},\sigma)\geq\rho>0$ a.s. for all t and all adversarial strategies.
\end{assumption}

Given a threshold $\alpha\in(0,1)$, define $\delta\in(0,0.5)$ satisfying $h(\delta)=\alpha$. Define the average belief of priors as $\bar{b}_t=\frac{1}{|\mathcal{A}|}\sum_a b_a(t)$.
\begin{theorem}\label{thm_homo}
Given a search strategy $\sigma$, if Assumptions~\ref{assumption} and~\ref{assumption_2} hold, and 
\begin{equation}\label{eq_supermartingale}
\frac{(1-\delta)p_f^\lambda}{\delta p_d^\lambda+(1-\delta)p_f^\lambda}\leq\frac{\rho}{C}\leq\frac{\rho\delta(1-p_f)^\lambda}{\delta(1-p_d)^\lambda+(1-\delta)(1-p_f)^\lambda},
\end{equation}
then the search game is $\alpha$-detectable for any $\alpha\geq h(\rho/C)$.
\end{theorem}

\noindent\textbf{Proof:}
Define $S_t=\sum_a \tilde{b}_a(t)$. Since the sum of prior beliefs equals the previous posterior sum, we have $\sum_a b_a(t)=S_{t-1}$. According to Lemma~\ref{lem_target}, the worst target strategy equalizes the next prior belief as $b_a(t)=\frac{1}{|\mathcal{A}|}S_{t-1}=\bar{b}_{t}$. Because the game is zero-sum, it is sufficient to show the $\alpha$-detectability by showing that the given strategy $\sigma$ can detect the target under this worst-case prior evolution.

We first analyze the infinite behavior of $S_t$. For a searched area, the expected posterior belief is
$$\begin{aligned}
&\mathbb{E}[\tilde{b}_a(t)|\mathcal{F}_{t-1},a=w_t]=\sum_{k=0}^\lambda  \frac{(P_D^\lambda(k))^2 \bar{b}_{t}}{P_{PD}^\lambda(k)},\\
&\mathbb{E}[\tilde{b}_a(t)|\mathcal{F}_{t-1},a\ne w_t]=\sum_{k=0}^\lambda \frac{P_D^\lambda(k)P_F^\lambda(k) \bar{b}_{t}}{P_{PD}^\lambda(k)}.
\end{aligned}
$$
Thus, their sum satisfies
$$\begin{aligned}
\mathbb{E}[S_{t}|\mathcal{F}_{t-1}]-S_{t-1}
=&\rho_t\mathbb{E}[S_{t}|\mathcal{F}_{t-1},w_{t}\in\mathcal{C}_{t}]+(1-\rho_t)\mathbb{E}[S_{t}|\mathcal{F}_{t-1},w_{t}\notin\mathcal{C}_{t}]\\
=&\bar{b}_{t}\left(\rho_t \sum_{k=0}^{\lambda}\frac{(P_D^\lambda(k))^2}{P_{PD}^\lambda(k)}+(C-\rho_t)\sum_{k=0}^{\lambda}\frac{P_D^\lambda(k)P_F^\lambda(k)}{P_{PD}^\lambda(k)}-C\right).
\end{aligned}
$$
With $\sum_{k=0}^{\lambda}P_D^\lambda(k)=1$, the terms in the above equation satisfy $$\sum_{k=0}^{\lambda}\frac{(P_D^\lambda(k))^2}{P_{PD}^\lambda(k)}-1=\sum_{k=0}^\lambda \frac{P_D^\lambda(k)(P_D^\lambda(k)-P_F^\lambda(k))}{P_{PD}^\lambda(k)}(1-\bar{b}_{t}),$$
$$\sum_{k=0}^{\lambda}\frac{P_D^\lambda(k)P_F^\lambda(k)}{P_{PD}^\lambda(k)}-1=\sum_{k=0}^\lambda \frac{P_D^\lambda(k)(P_F^\lambda(k)-P_D^\lambda(k))}{P_{PD}^\lambda(k)}\bar{b}_{t}.$$ Consequently, 
$$
\mathbb{E}[S_{t}|\mathcal{F}_{t-1}]-S_{t-1}=(\rho_t-C\bar{b}_{t})\bar{b}_{t}\sum_{k=0}^\lambda \frac{P_D^\lambda(k)(P_D^\lambda(k)-P_F^\lambda(k))}{P_{PD}^\lambda(k)}.
$$
Since $F(x)=\frac{1-x}{b+(1-b)x}$ is strictly convex for $b>0$, the summation term satisfies
$$
\sum_{k=0}^\lambda P_D^\lambda(k)\frac{P_D^\lambda(k)-P_F^\lambda(k)}{P_D^\lambda(k)\bar{b}_{t}+(1-\bar{b}_{t})P_F^\lambda(k)}
=\sum_{k=0}^\lambda P_D^\lambda(k)\frac{1-P_F^\lambda(k)/P_D^\lambda(k)}{\bar{b}_{t}+(1-\bar{b}_{t})P_F^\lambda(k)/P_D^\lambda(k)}> 0
.$$ 

Therefore, $\mathbb{E}[S_{t}|\mathcal{F}_{t-1}]>S_{t-1}$ if $\bar{b}_t<\frac{\rho_t}{C}$ and $\mathbb{E}[S_{t}|\mathcal{F}_{t-1}]<S_{t-1}$ if $\bar{b}_t>\frac{\rho_t}{C}$. Since $1\geq\rho_t\geq\rho$, we have $\mathbb{E}[S_t|\mathcal{F}_{t-1}]>S_{t-1}$ when $\bar{b}_t<\frac{\rho}{C}$ and $\mathbb{E}[S_t|\mathcal{F}_{t-1}]<S_{t-1}$ when $\bar{b}_t>\frac{1}{C}$. Together with the fact that $\mathbb{P}(\bar{b}_t<\frac{1}{C}|\bar{b}_{t-1}=\frac{1}{C}+\epsilon)>0$ and $\mathbb{P}(\bar{b}_t>\frac{1}{C}|\bar{b}_{t-1}=\frac{1}{C}-\epsilon)>0$ for some $\epsilon>0$ under all-negative or all-positive outcomes, the process cannot remain outside the interval $[\rho/C,1/C]$ after a finite time. Hence, there exists an infinite sequence $\{\tau_k\}$ such that $\bar{b}_{\tau_k}\in[\rho/C,1/C]$. Combining this with $\mathbb{P}(w_t\in\mathcal{C}_t|\mathcal{F}_{t-1})>\rho>0$ and the Borel-Cantelli lemma, there exists an infinite subsequence $\tilde{\tau}_k$ such that $w_{\tilde{\tau}_k}\in \mathcal{C}_{\tilde{\tau}_k}$ and $\bar{b}_{\tilde{\tau}_k}\in[\rho/C,1/C]$.

It remains to show that the $\alpha$-detection is reached in finite time. Suppose that, at time $\tilde{\tau}_k$, all observations at the target location are positive. Then
$
\tilde{b}_{w_{\tilde{\tau}_k}}(\tilde{\tau}_k|\lambda)=\frac{p_d^\lambda\bar{b}_{\tilde{\tau}_k}}{p_d^\lambda\bar{b}_{\tilde{\tau}_k}+p_f^\lambda(1-\bar{b}_{\tilde{\tau}_k})}$.
By~\eqref{eq_supermartingale}, we obtain $\tilde{b}_{w_{\tilde{\tau}_k}}(\tilde{\tau}_k|\lambda)\geq 1-\delta$. Similarly, if all observations at non-target covered areas are negative, then $\tilde{b}_{a}(\tilde{\tau}_k|0)\leq \delta$ for all covered areas except the target location. Since $\mathbb{P}(o_{a\ne w_{\tilde{\tau}_k}}=0,o_{w_{\tilde{\tau}_k}}=\lambda|\mathcal{F}_{\tilde{\tau}_k-1})$ is a positive constant, the Borel-Cantelli lemma implies that there exists a finite stopping time $\tau\in\{\tilde{\tau}_k\}$ such that $o_{a\ne w_{\tilde{\tau}_k}}=0$ for searched areas and $o_{w_{\tilde{\tau}_k}}=\lambda$. At this time, for searched areas, $\tilde{b}_{a\ne w_{\tilde{\tau}_k}}\leq\delta$ and $\tilde{b}_{w_{\tilde{\tau}_k}}\geq 1-\delta>0.5$, while for unsearched areas, $\tilde{b}_{a}=\bar{b}_t\leq\frac{\rho}{C}\leq \delta$ according to $\alpha\geq h(\rho/C)$. Therefore, the search game is $\alpha$-detectable.
\hfill$\square$

Theorem~\ref{thm_homo} provides a quantitatively sufficient condition for $\alpha$-detectability in the homogeneous case. Compared with Theorem~\ref{thm_martingale}, which gives a qualitative recurrence condition for entering the detection set, this result makes the condition explicit by reducing the belief evolution to the scalar average belief under the equalized worst-case prior. The lower and upper bounds in~\eqref{eq_supermartingale} ensure that, whenever the target is covered and the favorable observation event occurs, the target area obtains posterior belief at least $1-\delta$, while the other covered areas remain below $\delta$. Since $h(\delta)=\alpha$, this event satisfies both the entropy threshold and the unique dominant peak requirement in Definition~\ref{def-searchthreshold}.

The conditions are sufficient rather than necessary. The proof relies on the conservative event that all searchers covering the target report positive observations and all non-target covered areas report negative observations. Therefore, a strategy may still be $\alpha$-detectable even when conditions are not satisfied.

\begin{remark}
With numerical validation, \eqref{eq_supermartingale} is not difficult to satisfy. Consequently, the uniform lower bound of the coverage probability in Assumption~\ref{assumption_2} becomes the critical factor that determines $\alpha$-detectability. This insight into the critical role of a uniformly positive coverage probability guides the algorithm design, particularly when precise solutions become computationally prohibitive in high-dimensional multi-agent settings.
\end{remark}

Compared with existing works that do not consider false alarms \cite{Zhang2023, Alsamhi2022, Aggravi2021}, our result bridges the gap in the detectability criterion against false alarms. When false alarms are absent, i.e., $p_f=0$, the sufficient condition in Theorem~\ref{thm_homo} becomes trivial.

\begin{corollary}\label{cor_no_false_alarm}
Suppose Assumptions~\ref{assumption} and~\ref{assumption_2} hold. If $p_f=0$ and $C> 1$, then the search game is $\alpha$-detectable for any $\alpha>h(\frac{\rho}{C})$.
\end{corollary}
The proof is straightforward. When $p_f=0$, the left-hand side of~\eqref{eq_supermartingale} becomes $0$, while the right-hand side becomes $\frac{\rho}{1+(1-p_d)^\lambda}>\frac{\rho}{C}$ for $C>1$, which
completes the proof.

\section{Distributed Algorithm for Detectable Search Strategy}
\label{sec-algorithm}

In this section, we develop a computational framework for $\alpha$-detectable strategies for homogeneous searchers. To address the intractability caused by the high-dimensional action space and both the high-dimensional and continuous belief space inherent in the POSG, we decompose the searchers' collective optimization into local, area-based utility improvements by leveraging the aggregative potential game structure of the system. This allows searchers to adopt computationally efficient 1-step lookahead greedy updates that align with the global objective. For the target, we simplify the high-dimensional planning problem by shifting the utility to a KL divergence-based metric (Lemma~\ref{lem_KL}), which enables a state-space reduction to a one-dimensional average belief. On this basis, we propose the distributed algorithm, which also enables an $\alpha$-detectable strategy under certain conditions.

\subsection{Distributed Execution for Searchers}

As established in Lemma~\ref{lem_homo_prob}, in homogeneous settings, the Bayesian update depends only on the total number of detections rather than on which specific agent generated each report. Following the approach of the previous section, we replace each searcher's action with its induced sensing coverage, since the feasible motions uniquely determine which areas fall within the sensing range. Specifically, define $\lambda^i(a) \in \{0,1\}$ as the coverage indicator of searcher $i$ for area $a$, and let $\lambda^i = (\lambda^i(a))_{a \in \mathcal{A}}$ be the corresponding coverage vector, which provides an equivalent description of the searcher's action. The total coverage of area $a$ is then
$\lambda(a) = \sum_{i \in V} \lambda^i(a)$,
and $\lambda = (\lambda(a))_{a \in \mathcal{A}}$ refers to the aggregate coverage vector. For a fixed searcher $i$, let $\lambda^{-i}(a) = \sum_{j \neq i} \lambda^j(a)$ denote the coverage contributed by all other searchers except $i$.

Based on the utility function~\eqref{eq_expected-entropy}, the per-area expected utility, when expressed through the aggregate coverage count $\lambda(a)$, becomes
\begin{equation}\label{eq_Phi_def}
\Phi(\lambda(a), b_a)
=
\sum_{k=0}^{\lambda(a)}
P_{PD}^{\lambda(a)}(k)
h(b_a(t| o_a = k)),
\end{equation}
where $P_{PD}^{\lambda(a)}(k)$ is the subjective probability of observing exactly $k$ positive outcomes among the $\lambda(a)$ searchers covering area $a$, and $b_a(t| o_a = k)$ is the posterior belief. The definitions can be found in Lemma~\ref{lem_homo_prob}.

In the distributed setting, each searcher can only evaluate the effect of its actions within its reachable range $\bar{\mathcal{C}}^i_t$. Hence, the immediate local utility of searcher $i$ is
$$
u_i(\lambda^i,\lambda^{-i}|b(t|\pi))=\sum_{a \in \bar{\mathcal{C}}^i_t}
\Phi(\lambda^{-i}(a)+\lambda^i(a), b_a(t|\pi)).
$$
Since the global utility $\Psi(\lambda,b) = \sum_{a\in\mathcal{A}} \Phi(\lambda(a), b_a)$ is additively separable across areas and a unilateral action change by searcher $i$ affects only the areas in $\bar{\mathcal{C}}^i_t$, the game is an exact potential game with local payoff $u_i$ aligned with $\Psi$. Consequently, any greedy improvement of $u_i$ directly improves the global objective, and a pure NE is guaranteed to exist.

To avoid infinite-horizon planning over the continuous belief space, we adopt a 1-step lookahead approximation. This idea is motivated by the computational intractability of multi-step planning and the compounding prediction errors it entails. The resulting approximate action-value for searcher $i$ is
\begin{equation}\label{eq_lookahead_cost}
\begin{aligned}
\hat{U}_i(\lambda^i,\lambda^{-i},b(t)) = \mathbb{E}_{o}[ \sum_{a \in \bar{\mathcal{C}}^i_t} h(\tilde{b}_a(t|o))+ \xi \max_{\hat{\lambda}^i} u_i(\hat{\lambda}^i, \lambda^{-i}|b(t|o),\pi) ]\end{aligned}
\end{equation}
where $b_a(t|o)$ is the posterior belief after observing $o$. This local surrogate retains alignment with the global objective due to the exact potential structure established above, while being computationally tractable for distributed execution.

In a distributed potential game, purely greedy simultaneous updates can lead to cyclic behavior and fail to converge. We therefore adopt an inertia-based update~\cite{Arslan2017,Yongacoglu2024}, where each searcher retains its current action with probability $\beta$ and otherwise performs a greedy or exploratory update according to an $\epsilon$-greedy rule. Formally, searcher $i$ updates its action as: with probability $\beta$, it keeps the previous action, and with probability $1-\beta$, it chooses a new action according to an $\epsilon$-greedy rule. That is,
\begin{equation}\label{eq_action_update}
\lambda_t^i =
\begin{cases}
\lambda_{t-1}^i, & \text{with probability } \beta,\\
\tilde{\lambda}_t^i, & \text{with probability } 1-\beta,
\end{cases}
\end{equation}
where $\tilde{\lambda}_t^i$ is obtained by an $\epsilon$-greedy rule, namely, with probability $1-\epsilon$ take the greedy action  
$
\tilde{\lambda}_t^i = \arg\min_{{\lambda}^i \in \Lambda^i(v^i_{t-1})}
\hat{U}_i({\lambda}^i,\lambda_{t-1}^{-i},b(t)),
$
and with probability $\epsilon$ take a random action from the feasible action set $\Lambda^i(v^i_{t-1})$. Here, $\beta \in (0,1)$ is the inertia probability and $\epsilon \in (0,1)$ is the exploration rate. In the exploitation branch, the searcher performs a one-step greedy update with respect to the approximate utility $\hat{U}_i$. Due to the potential game structure, any decrease in $\hat{U}_i$ yields a decrease in the global approximations.

\subsection{Learning for Target Predictions}

From the perspective of searchers, the target's strategy is unknown and needs to be estimated in order to assess the prior belief $b(t+1)$. Estimating this strategy by working directly with the entire belief vector $b \in [0,1]^{|\mathcal{A}|}$ and the mixed strategy space $\Delta(Y)$ leads to excessive computational complexity. To handle this complexity, the following lemma establishes that the target's immediate reward depends on the entire belief vector only through its average $\bar{b}_t$. This observation allows us to approximate the target's value function using the average belief as the state variable.

\begin{lemma}\label{lem_KL}
The strategy that maximizes the target's immediate utility, i.e.,
$
\arg\max_{\pi \in \Delta(Y)} h(b(t|\pi)),
$
coincides with the solution of the following optimization
$$
\min_{\pi \in \Delta(Y)}
\sum_{a \in \mathcal{A}}
D_{\mathrm{KL}}( b_a(t|\pi) \| \bar{b}_t ),
$$
where $D_{\mathrm{KL}}(p \| q) = \sum_j p(j)\log\frac{p(j)}{q(j)}$ denotes the KL divergence and $\bar{b}_t$ is the average posterior belief.
\end{lemma}

\noindent\textbf{Proof.}
Expanding the KL divergence yields
$$
\begin{aligned}
\sum_{a \in \mathcal{A}} D_{\mathrm{KL}}( b_a(t|\pi) \| \bar{b}_t )
=& \sum_{a \in \mathcal{A}} \left(
b_a(t|\pi)\log\frac{b_a(t|\pi)}{\bar{b}_t}
+(1-b_a(t|\pi))\log\frac{1-b_a(t|\pi)}{1-\bar{b}_t}
\right) \\
=& |\mathcal{A}| h(\bar{b}_t) - h(b(t|\pi)),
\end{aligned}
$$
where we used the identity $\sum_a b_a(t|\pi) = |\mathcal{A}| \bar{b}_t$ and the fact that $\bar{b}_t$ is fixed at the moment of prediction. \hfill $\square$

Lemma~\ref{lem_KL} establishes an exact equivalence between the target's one-step entropy maximization and a KL-divergence minimization when the full predicted belief vector is available. With this equivalence, we employ the scalar average belief $\bar{b}$ as the reduced state rather than the entire belief vector to reduce computational complexity. This reduction retains the precision of the immediate utility with the subsequent value being approximated. Although this reduction brings additional approximation errors, it is valuable because directly solving an exact high-dimensional POSG is not practical.

During training, the value function $\hat{U}^T$ is learned against a weak searcher model that plays the no-lookahead greedy strategy, i.e., it simply minimizes the immediate entropy. While this opponent is not fully strategic, it provides a conservative and computationally affordable baseline for value learning. The immediate utility of the target in this setting is
\begin{equation}
\hat{u}^T(\bar{b},\pi)=-\sum_{a\in\mathcal{A}} D_{\mathrm{KL}}\big( b_a(t|\pi) \big\| \bar{b} \big),
\end{equation}
and by Lemma~\ref{lem_KL}, maximizing $\hat{u}^T$ is equivalent to maximizing the immediate entropy $h(b(t|\pi))$.

To further ensure tractability, we discretize the average belief and the target's mixed strategy space into finite sets $\mathbb{B}_b$ and $\hat{\Pi}$, respectively. For any $\bar{b}\in \mathbb{B}_b$, let $\hat{\bar{b}}$ denote the predicted next average belief after one step of search and evasion under the weak searcher model. The value function is then updated as
\begin{equation}\label{eq_target_bellman}
\hat{U}^T(\bar{b}) \leftarrow
\max_{\pi \in \hat{\Pi}}
( \hat{u}^T(\bar{b},\pi) + \xi\tilde{\hat{U}}^T(\hat{\bar{b}}) ),
\end{equation}
where $\tilde{\hat{U}}^T(\cdot)$ denotes the linear interpolation of current $\hat{U}^T$ on the grid $\mathbb{B}_b$. The update is iterated over all discretized states until convergence, yielding a look-up table $\hat{U}^T$ stored on the server. During execution, the server selects the target strategy greedily according to this value function.

\begin{remark}
Because the entire belief vector is reduced to the scalar average $\bar{b}$, the exact transition still requires knowledge of the full belief distribution, so approximation bias is unavoidable. The moderate loss of predictive fidelity is an acceptable trade-off here, as the reduced target-side value function is introduced to obtain a computationally feasible adversarial predictor.
\end{remark}

\subsection{Distributed Algorithm with Central Server}

The overall algorithm is structured in two phases. In the pre-training phase, the central server discretizes the reduced state and action spaces, and trains the target value function $\hat{U}^T$ via the Bellman update~\eqref{eq_target_bellman} against a weak searcher baseline. During the distributed search phase, each searcher distributedly selects its movement using~\eqref{eq_lookahead_cost} with local strategy exchanges, while the server performs belief aggregation and prior belief prediction. The procedure is summarized as follows.

\begin{algorithm}[!ht]
\caption{Server-assisted distributed algorithm}
\label{alg}
\begin{algorithmic}
\STATE \textbf{Phase 1 (Target Value Pre-Learning)}
\STATE Initialize $\mathbb{B}_b$, $\hat{\Pi}$ and $\hat{U}^T(\bar{b})$ for all $\bar{b}\in \mathbb{B}_b$.
\REPEAT
    \FOR{each $\bar{b}\in \mathbb{B}_b$}
        \STATE Train $\hat{U}^T$ according to the update rule~\eqref{eq_target_bellman}.
    \ENDFOR
\UNTIL{convergence or max steps}
\STATE \textbf{Phase 2 (Search Execution)}
\FOR{$t=0,1,2,\dots$}
    \STATE Server broadcasts predictions $b(t)$.
    \FORALL{searcher $i\in V$}
        \STATE Receive neighbors' $\lambda_{t-1}^j$. Choose $\lambda_t^i$ via~\eqref{eq_action_update} using $\hat{U}_i$ from~\eqref{eq_lookahead_cost}.
        \STATE Execute move and sense, send results to server.
    \ENDFOR
    \STATE Server updates $\tilde{b}(t)$ via~\eqref{eq_posterior} and gets their average $\bar{b}_{t+1}$.
    \STATE Server picks $\pi_t$ greedily according to $\hat{U}^T$ over $\hat{\Pi}$ as~\eqref{eq_target_bellman}, then predicts $b(t+1)$ by~\eqref{eq_transition}.
\ENDFOR
\end{algorithmic}
\end{algorithm}


Due to the approximations introduced in the algorithm, a direct convergence guarantee is difficult to establish. Instead, the following proposition summarizes a key property concerning equilibrium solutions.

\begin{proposition}\label{prop_equilibrium}
Let $\Sigma$ be the set of all searcher strategy profiles generable by Algorithm~\ref{alg} when the greedy action selection in~\eqref{eq_action_update} is replaced by a strict better response rule (i.e., $\tilde{\lambda}_t^i$ is chosen such that $\hat{U}_i(\tilde{\lambda}_t^i,\lambda_{t-1}^{-i},b(t)) < \hat{U}_i(\lambda_{t-1}^i,\lambda_{t-1}^{-i},b(t))$ whenever such an action exists). Then a Nash equilibrium strategy $\sigma^\ast$ satisfies $\sigma^\ast \in \Sigma$.
\end{proposition}

Proposition~\ref{prop_equilibrium} indicates that Algorithm~\ref{alg} does not exclude equilibrium strategies from the reachable strategy space. According to Theorem~\ref{thm_homo}, the essential requirement for $\alpha$-detectability is not precise optimality, but a uniformly positive lower bound on the probability of covering the target, i.e., $\mathbb{P}(w_t \in \mathcal{C}(t) | \mathcal{F}_{t-1}) \geq \rho > 0$ in Assumption~\ref{assumption_2}. As long as the search strategy maintains such a uniform lower bound on the coverage probability, the $\alpha$-detectability guarantee established in Theorem~\ref{thm_homo} remains valid even when the computed strategy deviates moderately from an exact equilibrium.

\section{Numerical Simulations}\label{sec-exp}

In this section, we present numerical simulations to illustrate the behavior of the proposed distributed algorithm and to examine the $\alpha$-detectability of the derived search strategies.

The simulation is conducted under the following settings. The search area is a $10 \times 10$ grid unless otherwise stated. Players' mobility is set to 1 grid per step. The detection threshold is set to $\alpha = 0.4$. Initially, no prior information about the target's location is assumed, i.e., $b_a(0)=0.5$ for all areas $a$. For all entropy plots in this section, each trajectory is shown only up to the first step at which the detection threshold $\alpha$ is crossed, or up to 300 steps if the threshold is never reached.

\subsection{Validation of Distributed Algorithm}

We first examine whether the proposed distributed algorithm can produce effective search strategies for a given perceptual setting $p_d=0.8$ and $p_f=0.1$ with 3 searchers. Fig.~\ref{fig:entropy_single} and Fig.~\ref{fig:heatmaps} present the same single run from two complementary perspectives. Fig.~\ref{fig:entropy_single} shows the total posterior entropy $\sum_{a\in\mathcal{A}} h(\tilde{b}_a(t))$, which drops sharply from its initial high value, then rebounds and subsequently goes through several cycles of decline and increase. To provide a more direct assessment of whether the searchers have truly located the target, we additionally mark the time steps at which the area of the highest posterior belief, i.e., $\arg\max_a \tilde{b}_a(t)$, coincides with the true target position. Fig.~\ref{fig:heatmaps} depicts the evolution of the belief map and the positions of the agents at four selected steps of the same run. 

\begin{figure}[!t]
    \centering
    \includegraphics[width=0.6\linewidth]{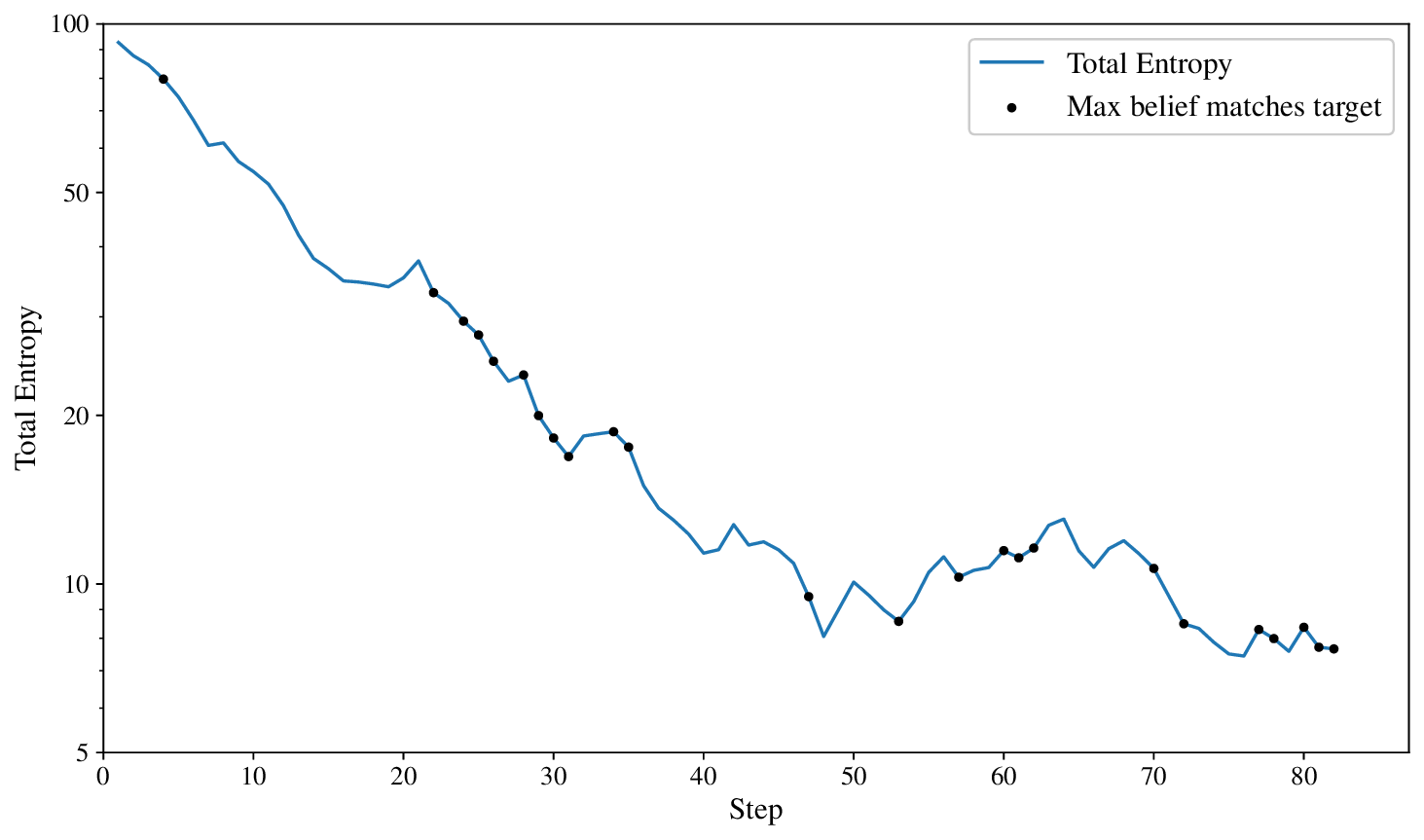}
    \caption{Total posterior entropy for $m=3$ searchers under moderate accuracy ($p_d=0.8$, $p_f=0.1$). At step 82, it reaches $\alpha$-detectability.}
    \label{fig:entropy_single}
\end{figure}

\begin{figure}[!t]
    \centering
    \subfloat[Step 10]{\includegraphics[width=0.225\linewidth]{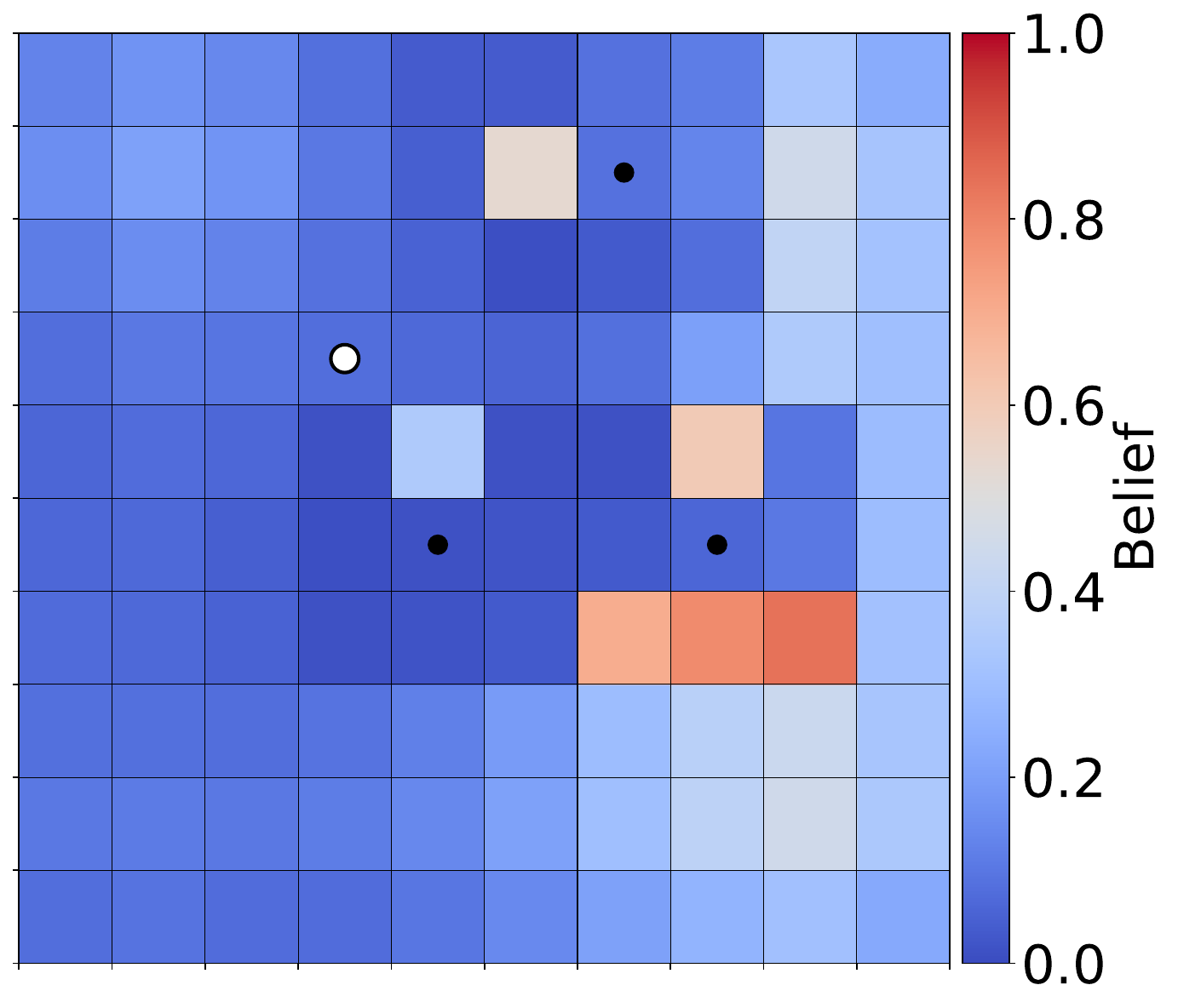}\label{fig:heatmap_1}}
    \subfloat[Step 30]{\includegraphics[width=0.225\linewidth]{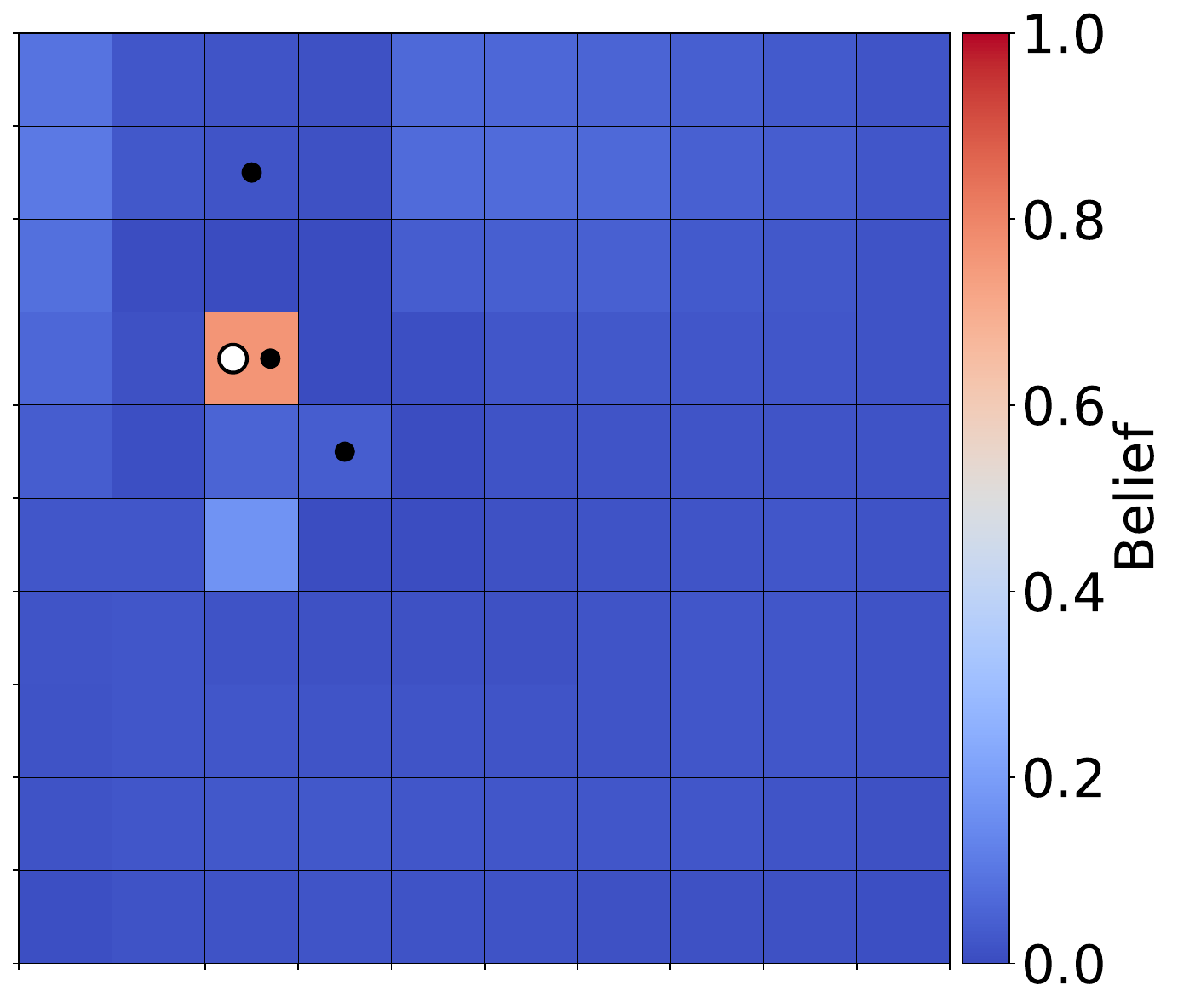}\label{fig:heatmap_2}}
    \subfloat[Step 40]{\includegraphics[width=0.225\linewidth]{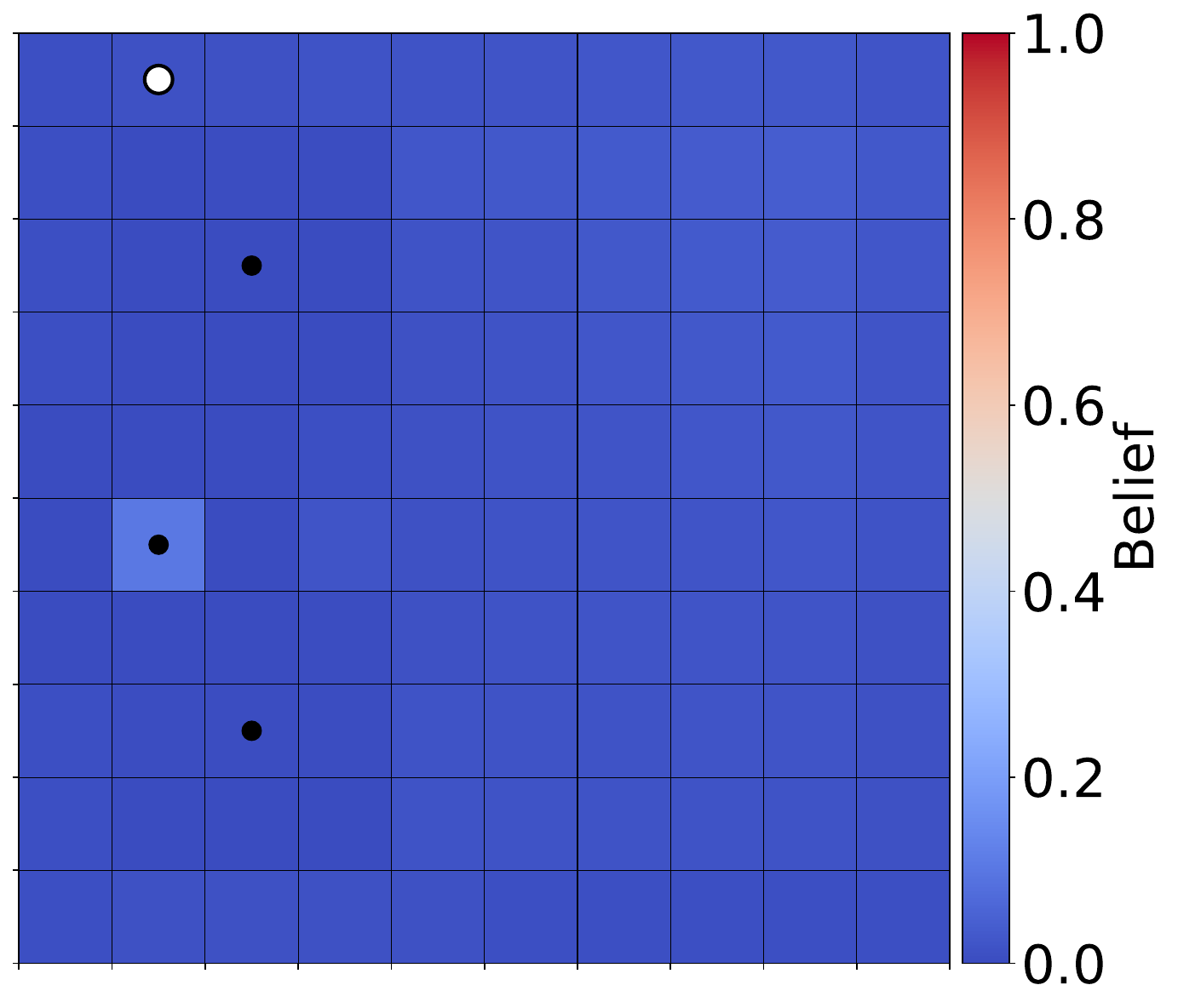}\label{fig:heatmap_3}}
    \subfloat[Step 82]{\includegraphics[width=0.225\linewidth]{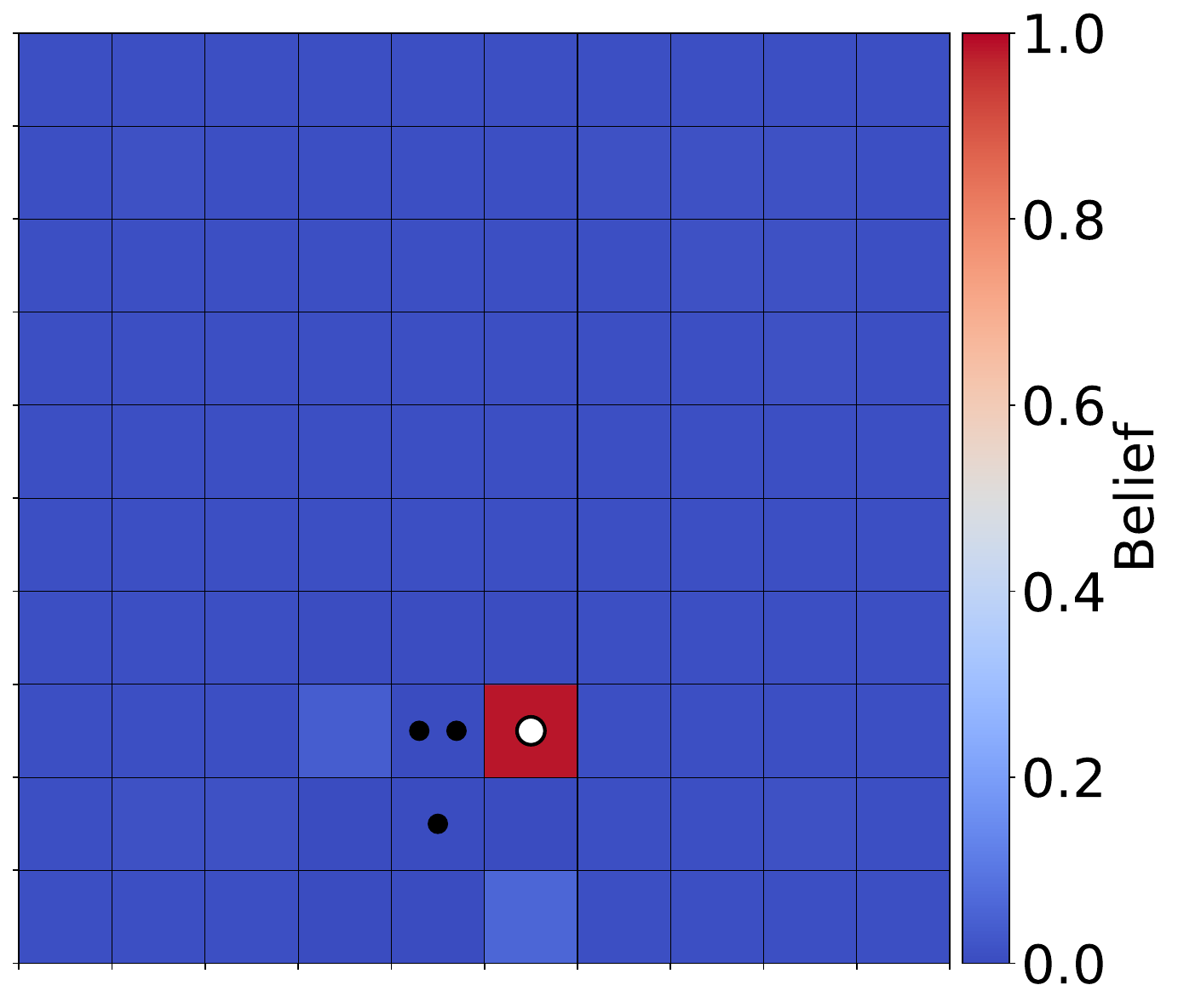}\label{fig:heatmap_4}}
    \caption{Belief maps and agent positions for the run in Fig.~\ref{fig:entropy_single}.}
    \label{fig:heatmaps}
\end{figure}

Notably, black dots first persistently appear around step 30, which is also consistent with its heatmap, indicating that the searchers have already located the target in terms of maximum posterior belief. However, the total entropy at this stage remains well above the threshold $\alpha=0.4$, revealing that a correct maximum-likelihood estimate does not yet constitute reliable detection. This gap precisely illustrates the necessity of the joint criterion in Definition~\ref{def-searchthreshold}: the entropy threshold serves as a reliability measure that complements the peak location, thereby highlighting the significance of $\alpha$-detectability. Subsequently, the entropy rebounds and the process enters a phase of belief vanishing, where no area maintains a confidence above $0.5$ and the posterior mass diffuses across the grid. This phenomenon confirms that guaranteed detection critically relies on the recurrent coverage condition in Assumption~\ref{assumption_2}; without a persistent positive probability $\rho$ of covering the target, the evader can escape the sensing range for extended periods, causing information loss to dominate. Around step 80, the searchers re-establish coverage through continued exploration, and the black dots reappear. Leveraging renewed observations, the entropy undergoes a final decline and drops below $\alpha$ at step 82 while a unique dominant peak exceeding $0.5$ persists at the true target location, thereby achieving $\alpha$-detectability.

\begin{figure}[!t]
    \centering
    \includegraphics[width=0.6\linewidth]{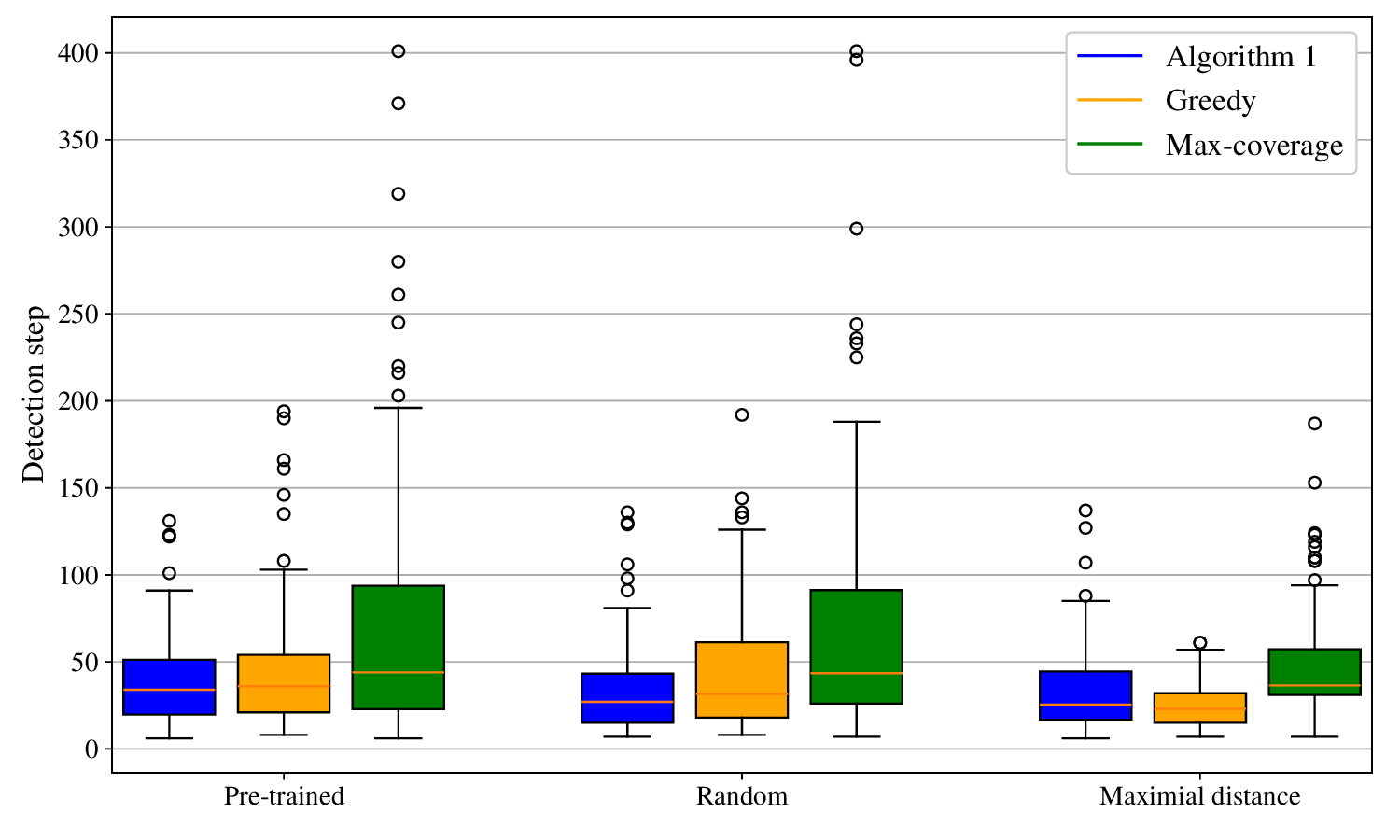}
    \caption{Box plots of detection steps for all combinations of searcher and target strategies (high false alarm case: $5\times5$ grid, $p_d=0.8$, $p_f=0.3$, $m=3$, 100 runs each).}
    \label{fig:baseline_boxplot}
\end{figure}

To further assess the effectiveness of the proposed algorithm, especially facing false alarms, we compare it with another two heuristic searcher strategies, greedy (maximizing the highest expected posterior probability) and max-coverage (maximizing the cumulative count of coverage), against three target strategies: (1) Algorithm~\ref{alg}, (2) random, and (3) max-distance (from searchers).

We conduct 100 independent runs for each of the nine combinations under a high false alarm rate condition: $5\times5$ grid, $p_d=0.8$, $p_f=0.3$, $m=3$, with a maximal step of 400. This design is intended to emphasize the impact of false alarms on different search strategies, so the search task itself is kept relatively easy and the other parameters remain fairly benign. Fig.~\ref{fig:baseline_boxplot} presents the detection-step statistics as box plots. The results demonstrate that the proposed algorithm achieves a relatively low median detection time and a relatively low variance, indicating robust and reliable performance against noisy perceptions. The learned NE target proves to be the most challenging adversary for all search strategies, validating that the value iteration captures a strong evasion policy.

\subsection{Verification of Detectability}

We now turn to the verification of $\alpha$-detectability from an empirical perspective. Fig.~\ref{fig:entropy_perception} shows representative total entropy trajectories under three sensing configurations: (i) high accuracy ($p_d=0.95$, $p_f=0.02$), (ii) moderate accuracy ($p_d=0.8$, $p_f=0.1$), and (iii) low accuracy ($p_d=0.6$, $p_f=0.2$), with $m=3$ searchers. While Fig.~\ref{fig:entropy_searchers} shows trajectories under the moderate regime for varying team sizes $m=2,3,4$. Fig.~\ref{fig:detect_steps_accuracy} and Fig.~\ref{fig:detect_steps_searchers} present the corresponding detection step statistics as box plots, respectively.

These figures collectively imply that $\alpha$-detectability is achieved only when factors are sufficiently favorable. In Fig.~\ref{fig:entropy_perception} and Fig.~\ref{fig:detect_steps_accuracy}, only the hard case fails to reach $\alpha$-detectability within 300 steps; higher accuracy yields shorter detection times and smaller variance. Likewise, in Fig.~\ref{fig:entropy_searchers} and Fig.~\ref{fig:detect_steps_searchers}, only $m=2$ frequently fails to achieve $\alpha$-detectability, whereas larger teams succeed; increasing the team size accelerates detection and reduces variability.

This finding is consistent with Theorem~\ref{thm_homo}. When the SNR is too low (hard case) or the coverage is too small ($m=2$), the inequality in Theorem~\ref{thm_homo} and the uniform lower bound on covering the target may no longer hold. Consequently, the noise from false alarms and missed detections or uncovered stages dominates the dynamics, so that either the total entropy ceases to decrease or the belief remains trapped in vanishing.

\begin{figure}[!t]
    \centering
    \includegraphics[width=0.6\linewidth]{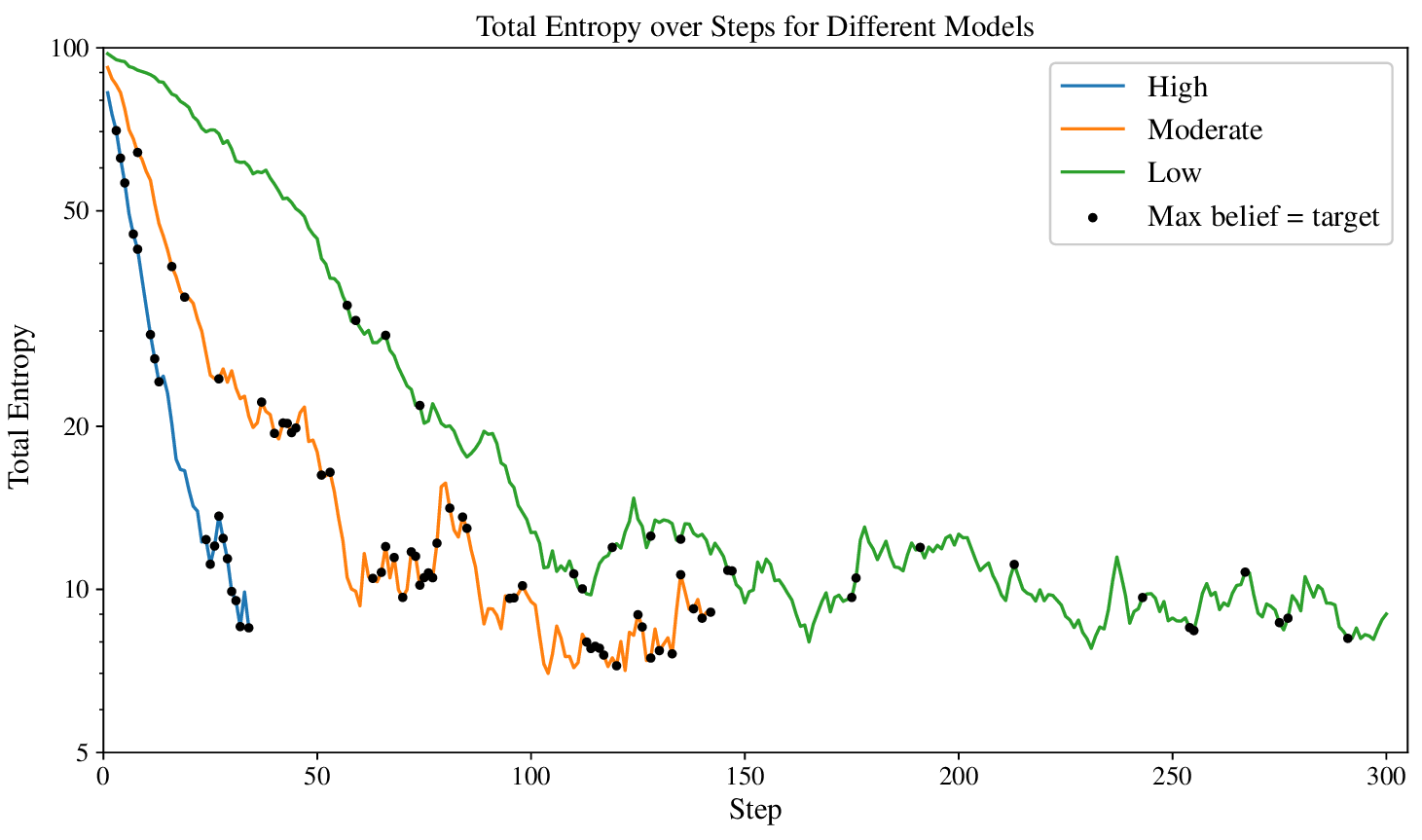}
    \caption{Total entropy over steps under each sensing accuracy. The hard configuration fails to achieve detectability within the horizon.}
    \label{fig:entropy_perception}
\end{figure}

\begin{figure}[!t]
    \centering
    \includegraphics[width=0.6\linewidth]{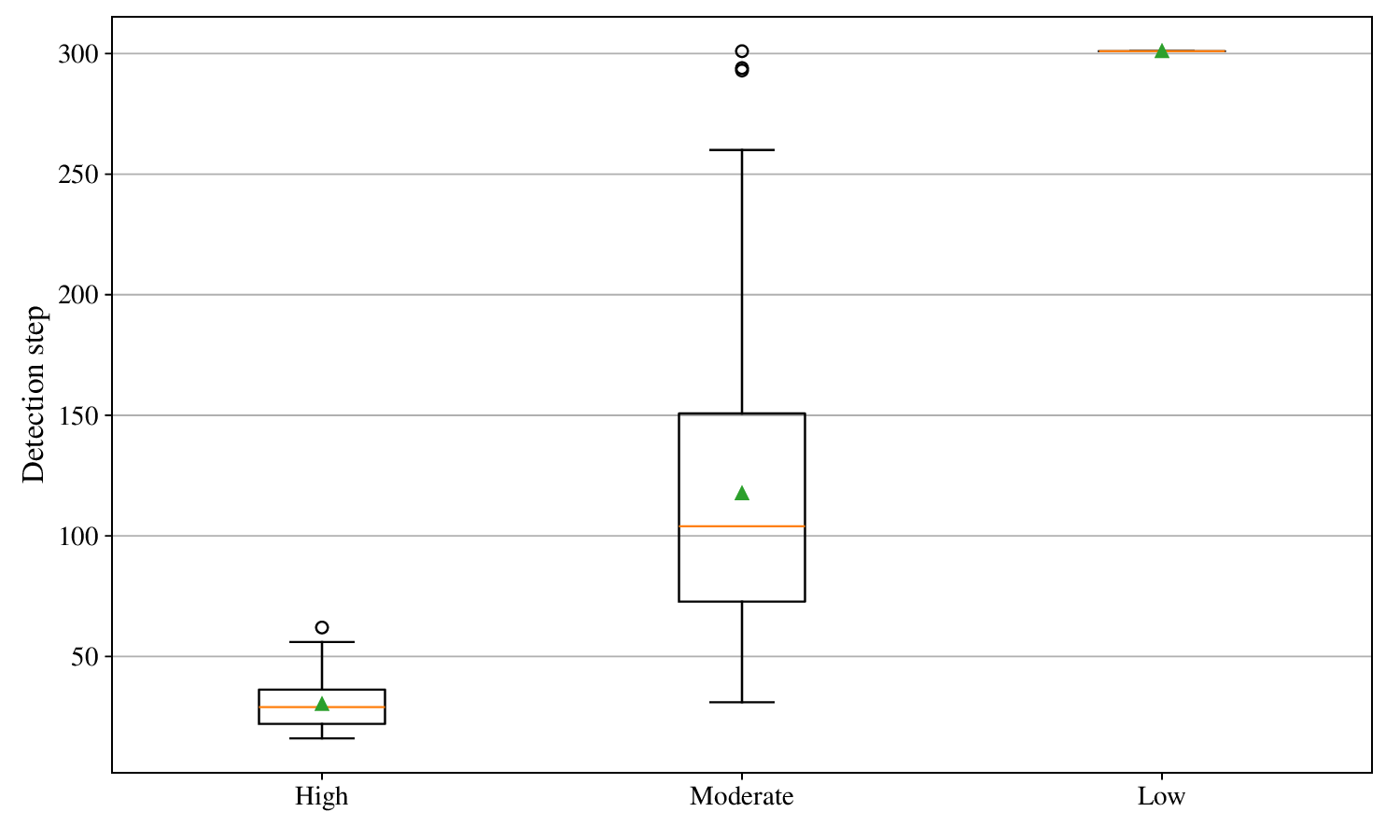}
    \caption{Box plots of detection steps for different sensing accuracies (High: $p_d=0.95$, $p_f=0.02$, Moderate: $p_d=0.8$, $p_f=0.1$, Low: $p_d=0.6$, $p_f=0.2$, 100 runs each). The hard one fails to detect within 300 steps.}
    \label{fig:detect_steps_accuracy}
\end{figure}

\begin{figure}[!t]
    \centering
    \includegraphics[width=0.6\linewidth]{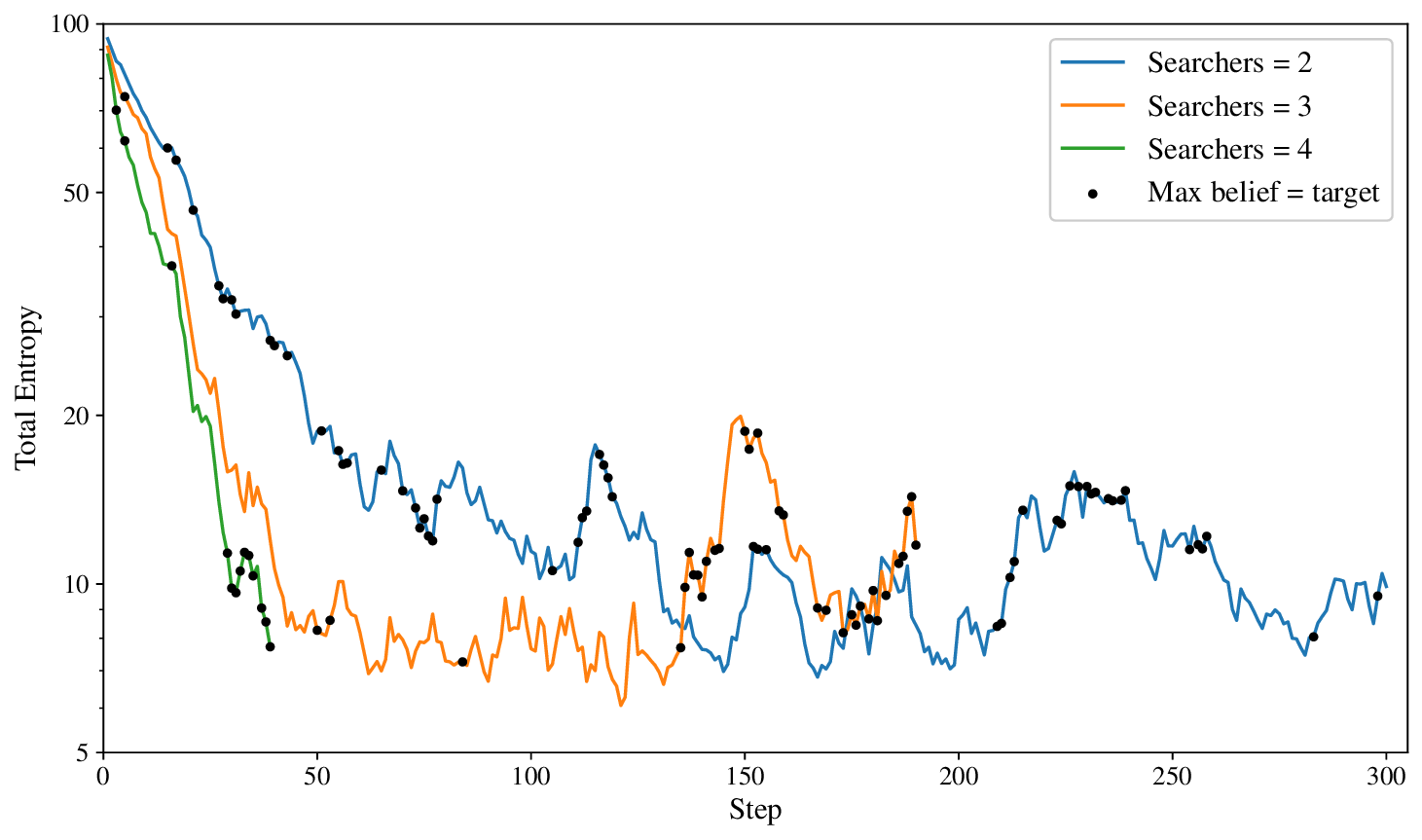}
    \caption{Total entropy over steps under different searcher numbers. The $m=2$ configuration fails to achieve detectability within the horizon.}
    \label{fig:entropy_searchers}
\end{figure}

\begin{figure}[!t]
    \centering
    \includegraphics[width=0.6\linewidth]{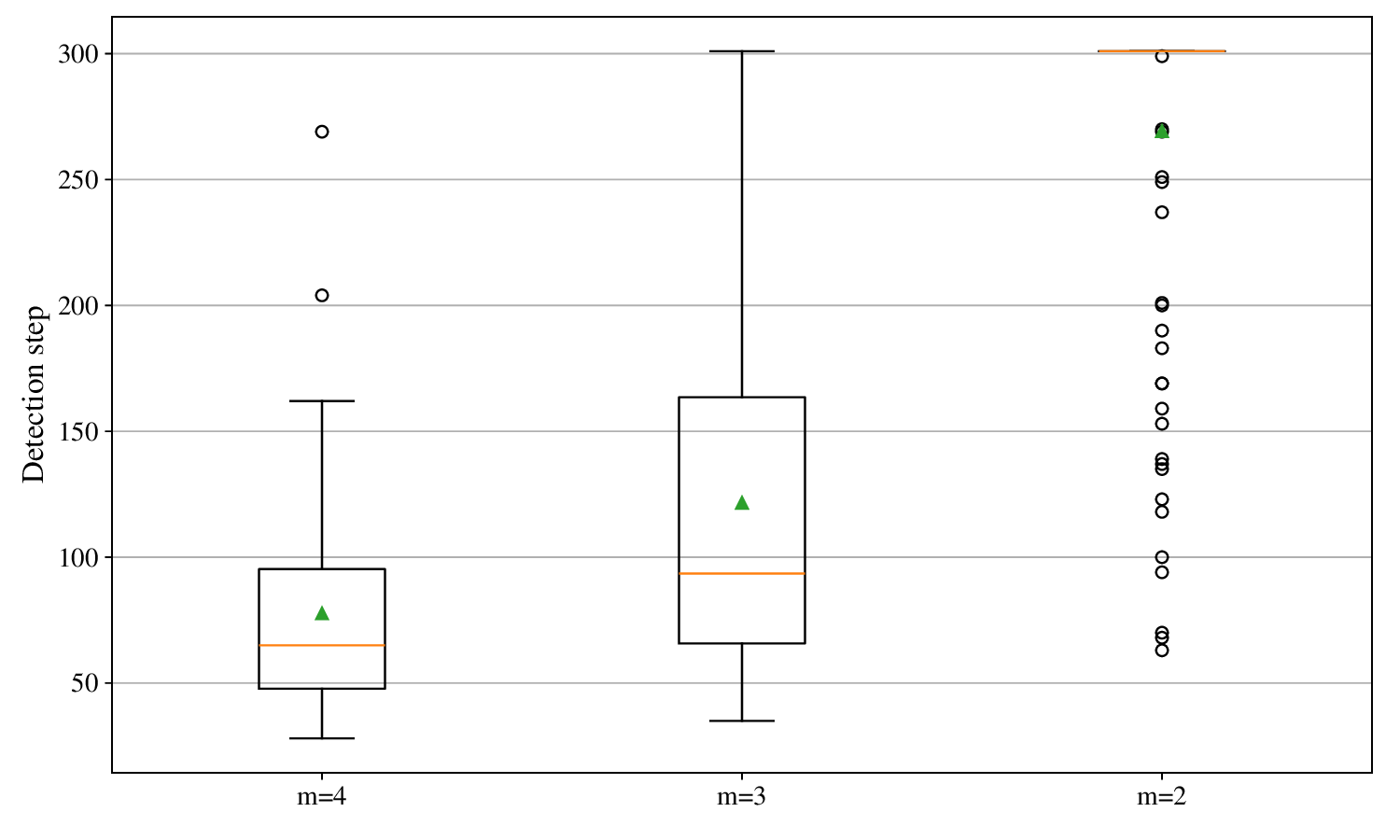}
    \caption{Box plots of detection steps for different number of searchers ($m=2,3,4$).}
    \label{fig:detect_steps_searchers}
\end{figure}

Overall, the results in this subsection provide empirical support for the detectability criteria derived in Section~\ref{sec-detectable}. Although these numerical results cannot serve as a formal proof of $\alpha$-detectability for every specific scenario, they demonstrate that detectability is contingent on sufficient sensing quality and number of searchers: the hard accuracy setting and the number of searchers $m=2$ fail to achieve detection within the 300-step horizon because sustained belief vanishes, whereas more favorable configurations reliably satisfy the criterion.

\section{Conclusions}\label{sec-conclusion}

In this work, we addressed mobile target search under uncertain perceptions by adopting a distributed partially observable stochastic game (POSG) approach. We proposed the concept of $\alpha$-detectability which determines whether a search strategy can guarantee eventual detection despite false alarms and missed detections, and established sufficient conditions from both qualitative and quantitative perspectives using conditional recurrence analysis of the belief process. To compute an $\alpha$-detectable search strategy under the high-dimensional belief and action spaces, we designed a distributed algorithm that leverages the aggregative potential structure for local searcher-side updates and a KL-divergence-based reduction for target-side prediction. Numerical simulations validated the effectiveness of the algorithm and provided empirical support for our detectability analysis.

\bibliographystyle{ieeetr}
\bibliography{ref.bib}

\end{document}